\documentclass[11pt]{article}
\usepackage{fullpage}
\usepackage{natbib}
\usepackage{hyperref}
\usepackage{url}
\usepackage{booktabs}       
\usepackage{amsfonts}       
\usepackage{nicefrac}       
\usepackage{microtype}      
\usepackage[table,dvipsnames]{xcolor}
\usepackage{amsmath}
\usepackage{algorithmic}
\usepackage[ruled,vlined]{algorithm2e}
\usepackage{graphicx}
\graphicspath{{arxiv/}}
\usepackage{float}
\usepackage{subcaption}
\usepackage{enumitem}
\newcommand{\aname}{SORL}
\newcommand{\pponame}{SO-PPO}
\newcommand{\grponame}{SO-GRPO}
\newcommand{\greenchar}[1]{\colorbox{green!30}{\strut #1}}
\newcommand{\redchar}[1]{\colorbox{red!30}{\strut #1}}

\makeatletter
\def\thanks#1{\protected@xdef\@thanks{\@thanks
        \protect\footnotetext{#1}}}
\makeatother

\usepackage{arydshln}    
\usepackage[most]{tcolorbox}
\usepackage{multirow}
\usepackage[utf8]{inputenc} 
\usepackage[T1]{fontenc}    
\usepackage{dsfont}
\usepackage{amssymb}
\usepackage{mathtools}
\usepackage{amsthm}
\usepackage[normalem]{ulem}

\newtheorem{lemma}{Lemma}[section]

\title{Stabilizing Off-Policy Training for Long-Horizon LLM Agents via Turn-Level Importance Sampling and Clipping-Triggered Normalization}

\author{
Chenliang Li\textsuperscript{1}$^{*}$\thanks{$^{*}$This work was done during an internship at GE HealthCare, Bellevue, WA.}\thanks{\raggedright Corresponding authors: Chenliang Li (\texttt{chenliangli@tamu.edu}) and Adel Elmahdy (\texttt{adel.elmahdy@gehealthcare.com}).},
Adel Elmahdy\textsuperscript{2},
Alex Boyd\textsuperscript{2},
Zhongruo Wang\textsuperscript{3},
Siliang Zeng\textsuperscript{4},\\
Alfredo Garcia\textsuperscript{1},
Parminder Bhatia\textsuperscript{2},
Taha Kass-Hout\textsuperscript{2},
Cao Xiao\textsuperscript{2},
Mingyi Hong\textsuperscript{4} \\[4pt]
\textsuperscript{1}Texas A\&M University \quad
\textsuperscript{2}GE HealthCare \quad \\
\textsuperscript{3}Independent Researcher \quad
\textsuperscript{4}University of Minnesota\\[4pt]
}

\hypersetup{pdftitle={Stabilizing Off-Policy Training for Long-Horizon LLM Agents via Turn-Level Importance Sampling and Clipping-Triggered Normalization}, pdfauthor={Chenliang Li et al.}}

\begin{document}
\maketitle

\begin{abstract}
Reinforcement learning (RL) algorithms such as PPO and GRPO are widely used to train large language models (LLMs) for multi-turn agentic tasks. However, in off-policy training pipelines, these methods can exhibit unstable optimization dynamics and are prone to performance collapse. Through empirical analysis, we identify two fundamental sources of instability in this setting: (1)~a granularity mismatch between token-level policy optimization and turn-structured interactions, and (2)~high-variance and unreliable gradient updates induced by off-policy importance sampling and inaccurate advantage estimation. To address these challenges, we propose \aname, \underline{S}tabilizing \underline{O}ff-Policy \underline{R}einforcement \underline{L}earning for Long-Horizon Agent Training. \aname~introduces mechanisms that align policy optimization with the structure of multi-turn interactions and adaptively suppress unreliable off-policy updates, yielding more conservative and robust learning dynamics. Within this framework, we instantiate two stabilized algorithms: \pponame~and \grponame. Both algorithms are designed to mitigate gradient variance and prevent optimization collapse without requiring careful early stopping or heuristic tuning. {We evaluate \pponame~and \grponame~on benchmarks spanning open-domain QA, multi-hop QA, and medical multiple-choice QA, and further assess their transfer to asynchronous RL for mathematical reasoning by training on DAPO-Math-17k and validating on AIME-2024.} Experimental results show that both methods consistently prevent the training instabilities and performance collapses observed in standard PPO and GRPO, maintain lower clipping ratios and more stable optimization trajectories, and achieve superior or comparable task performance. These results demonstrate that \aname{} provides a practical, scalable, and general framework for stabilizing reinforcement learning in multi-turn LLM agent training and asynchronous RL for mathematical reasoning. Code is available at \href{https://github.com/Cloud0723/SORL/tree/main}{here}
\end{abstract}

\section{Introduction}
Reinforcement learning (RL) has significantly advanced the reasoning capabilities of large language models, enabling strong performance in domains such as mathematical problem solving \citep{jaech2024openai,liu2024deepseek,yu2025dapo} and code generation \citep{el2025competitive,cui2025process}. Beyond these applications, RL has also shown promise in more agentic settings such as tool learning \citep{qian2025toolrl,feng2025retool}, where models learn to invoke external tools (e.g., web search engines), execute actions, and interact with real-world environments. Recent systems such as DeepSeek-V3 \citep{liu2024deepseek} and Kimi K2 \citep{team2025kimi} have achieved state-of-the-art performance on both mathematical reasoning (e.g., AIME, MATH-500) and agentic benchmarks \citep{jimenez2024swe}. Despite these successes, the computational demands of multi-turn RL training pose significant practical challenges. Such gains rely on a large number of interaction samples, which are costly due to the large number of rollouts and multi-turn tool calls required during training. In practice, hardware and memory limits force each batch of collected samples to be split into several mini-batches \citep{schulman2017proximal} and updated sequentially. This naturally induces a hybrid update mechanism where later updates become increasingly off-policy \citep{chen2023sufficiency}. To maximize sample efficiency under computational constraints, practitioners often adopt off-policy pipelines, and the resulting distribution mismatch is typically corrected using importance sampling \citep{nachum2017trust}.

\begin{figure*}[t]
    \centering
    \includegraphics[
      width=0.88\linewidth,
      trim=0cm 2.5cm 0cm 2.0cm,
      clip
    ]{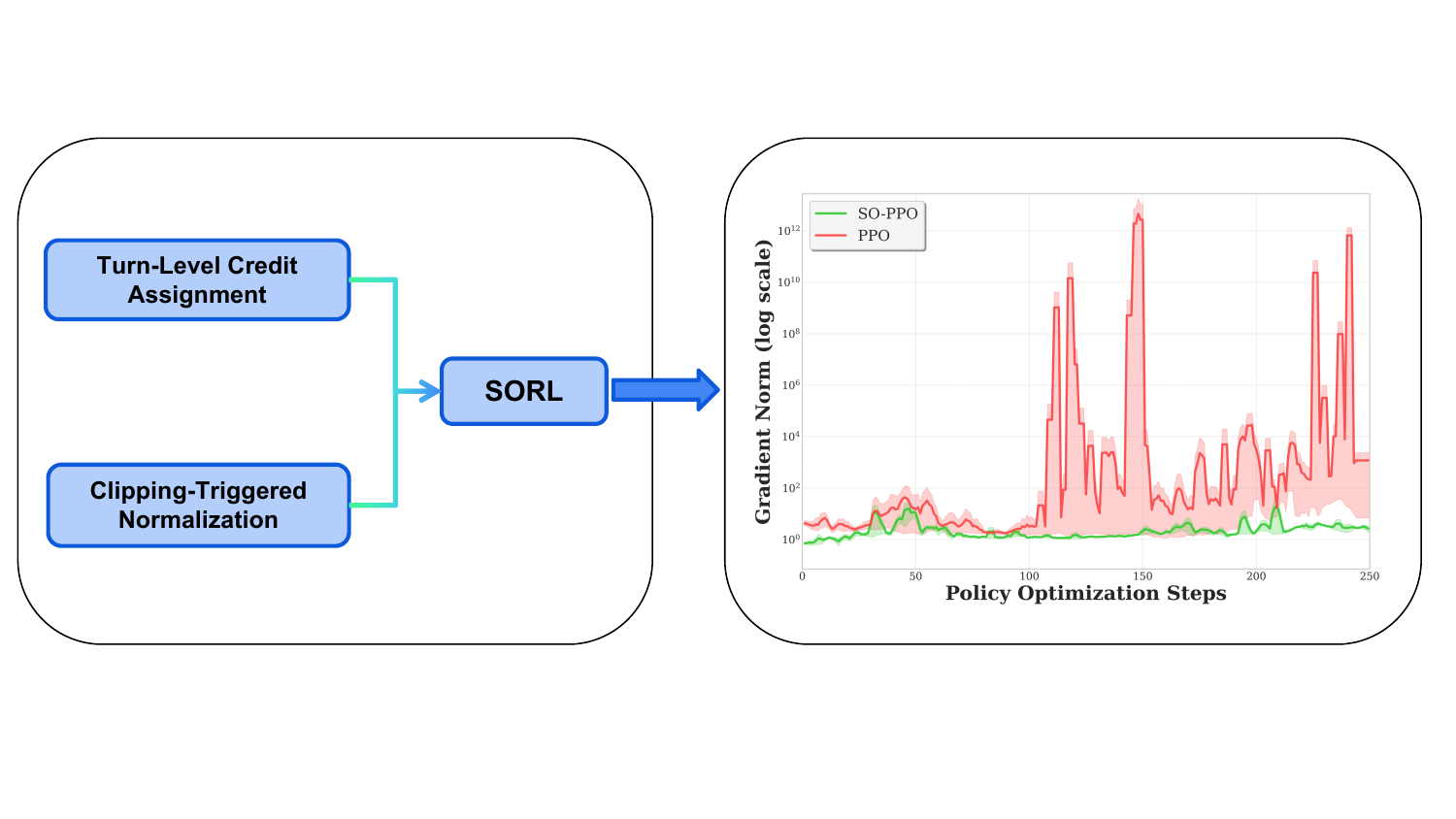}
    \caption{Illustration of the proposed \aname~framework. The left panel shows the two key components of \aname: turn-level importance sampling (which induces turn-level credit assignment) and clipping-triggered normalization. The right panel reports policy gradient norms during training, demonstrating that the proposed framework effectively suppresses extreme gradient spikes and leads to more stable optimization dynamics.}
    \label{fig:framework_STPPO}
\end{figure*}

The shift to off-policy methods, while necessary for computational efficiency, introduces high variance and can destabilize training \citep{munos2016safe,precup2000eligibility}. To address this challenge, proximal policy optimization (PPO) \citep{schulman2017proximal} constrains policy updates with a clipped surrogate objective. By limiting the impact of importance ratios, PPO stabilizes learning even when mini-batch reuse makes later updates effectively off-policy. Building on this principle, recent studies \citep{zheng2025group,roux2025tapered,yan2025learning} refine importance sampling and clipping strategies to further enhance stability and efficiency.

Despite these advances, PPO training on multi-turn tasks still suffers from two core deficiencies:
(1)~token-level credit assignment is misaligned with the natural granularity of multi-turn environments \citep{zeng2025reinforcing}, where reasoning unfolds through distinct turn-level stages (e.g., problem analysis, query formulation, information processing), and
(2)~off-policy updates increasingly dominate as training progresses, causing the data distribution to drift further from the current policy. This growing off-policy mismatch exacerbates the variance of importance sampling and undermines the reliability of advantage estimates, particularly for out-of-distribution tokens \citep{dorka2022adaptively}. As a result, gradient variance increases over time, leading to unstable updates and potential policy collapse even under clipping.

To address these challenges, we propose \textbf{\aname}, a stabilizing off-policy RL framework that combines turn-level importance sampling for structure-aware credit assignment with clipping-triggered normalization to suppress unreliable off-policy updates, yielding more stable optimization. \textbf{Based on this design, our key contributions are as follows:}
\begin{enumerate}[leftmargin=*,itemsep=3pt,topsep=0pt,parsep=0pt,partopsep=0pt]
    \item We empirically diagnose why vanilla PPO becomes unstable when applied to multi-turn LLM agent training under off-policy updates. Our analysis identifies two root causes unique to the agentic setting: (i) the granularity mismatch between token-level optimization and turn-level interactions, and (ii) the accumulation of variance from unreliable off-policy samples, where state--action pairs are poorly evaluated.

    \item We propose a unified off-policy reinforcement learning framework that aligns policy optimization with the structure of multi-turn reasoning and stabilizes training under increasing off-policy mismatch. Specifically, the framework integrates turn-level importance sampling with a clipping-triggered normalization mechanism to suppress unreliable updates, yielding more conservative gradient estimates and significantly improved training stability in long-horizon agent learning.

    \item We demonstrate that the proposed framework is algorithm-agnostic and can be readily instantiated in other reinforcement learning methods beyond PPO, such as GRPO. Experimental results show that the resulting stabilized variant substantially improves training stability and prevents performance collapse compared to standard GRPO, highlighting the generality and practical effectiveness of our approach.
\end{enumerate}

These contributions collectively establish \textbf{\aname} as a framework for stabilizing off-policy reinforcement learning in multi-turn LLM agent training, integrating theoretical insights with algorithmic design. By addressing both the granularity mismatch in credit assignment and the instability induced by increasingly off-policy updates, our approach enables more stable and conservative optimization. Empirically, \aname{} mitigates the training collapse observed in large-scale models and consistently improves task performance, offering a practical and scalable solution for reinforcement learning with multi-turn LLM agents. Related work is discussed in Section~\ref{sec:related_work}.
\section{Related Work}\label{sec:related_work}
\subsection{Reinforcement Learning with LLM Agents}
Recent advances in reinforcement learning for large language models (LLMs) have primarily progressed along two main directions: (1) feedback-driven alignment and policy optimization, and (2) agentic tool use with long-horizon reasoning.

On the alignment side, RLHF~\citep{ouyang2022training} translates human or rule-based preferences into reward signals and optimizes policies using policy gradient methods, as exemplified by InstructGPT. Subsequent approaches such as direct preference optimization (DPO)~\citep{rafailov2023direct} avoid explicit reward modeling, while other RL algorithms—including PPO~\citep{schulman2017proximal}, GRPO~\citep{shao2024deepseekmath}, and RLOO~\citep{ahmadian2024back}—have been adapted to large models to balance stability and efficiency.

On the agentic side, methods such as ReAct~\citep{yao2023react} and GIVE~\citep{he2024give, he2025self} integrate reasoning and acting through interleaved thought–action steps, while Reflexion~\citep{shinn2023reflexion} leverages self-reflection and memory to enhance multi-turn decision making. Toolformer~\citep{schick2023toolformer} demonstrates self-supervised learning of tool usage, and WebGPT~\citep{nakano2021webgpt} introduces browser-based agents trained with human feedback. More recent systems, such as Search-R1~\citep{jin2025search}, WebAgent-R1~\citep{wei2025webagent}, and WebRL~\citep{qi2024webrl}, extend this line of research to end-to-end reinforcement learning for web-based agents, where models are trained to perform retrieval, reasoning, and action within real-world interactive environments.

\subsection{ {Off-Policy Reinforcement Learning with LLM Agents}}
Off-policy reinforcement learning methods differ fundamentally from on-policy approaches by allowing the agent to learn from data generated by a policy different from the one currently being improved. This characteristic enables greater flexibility and data efficiency, as off-policy algorithms can reuse past experiences or observations collected under different policies. This property is particularly important in scenarios where data collection is costly or limited.

Recent work has extended these ideas to policy gradient methods tailored for off-policy learning in complex environments. The tapered off-policy REINFORCE (TOPR) algorithm \citep{roux2025tapered} builds on the classic REINFORCE method \citep{williams1992simple} by introducing a tapered importance sampling scheme for the policy gradient. This tapering reduces the variance and instability that commonly arise in naive off-policy gradient methods, as it asymmetrically adjusts the importance weights to better balance learning from positive and negative examples. However, because TOPR relies heavily on trajectory-level rewards, it is not well-suited for scenarios that require more fine-grained credit assignment. BAPO~\citep{xi2026bapo} further stabilizes off-policy policy optimization by adaptively adjusting the clipping bounds to balance the gradient contributions of positive- and negative-advantage samples. This design mitigates optimization bias and helps preserve policy entropy when training with stale or replayed data.

Another research line explores off-policy extensions of GRPO. ARPO~\citep{lu2025arpo} incorporates a~replay buffer into GRPO’s sampling process, mixing on-policy and replayed samples to alleviate the zero-advantage issue, while RePO~\citep{li2025repo} reuses past outputs as off-policy samples with tailored replay strategies (e.g., recency-based, reward-oriented) to improve data efficiency and stability. Although these methods enhance sample utilization, they still require substantial training time, and their trajectory-level credit assignment often leads to slower convergence on multi-turn~tasks.

\subsection{Reinforcement Learning in  {Multi-Turn Tasks}}

Recent advances in LLMs have increasingly leveraged reinforcement learning to enhance multi-turn reasoning and tool-use capabilities \citep{chen2025learning,cheng2026stop,li2025torl}.
The work on turn-level credit assignment by \citet{zeng2025reinforcing} introduces a novel RL framework that models multi-turn interactive tasks as Markov decision processes, enabling fine-grained credit allocation to individual reasoning and tool-use steps within a trajectory.
Complementing this, the Search-R1 framework proposed by \citet{jin2025empirical,jin2025search} similarly harnesses RL to train LLMs to reason about when and how to use external search engines interactively.
By framing the reasoning and search process as an end-to-end decision-making problem, Search-R1 encourages LLMs to generate informative queries and iteratively refine answers based on real-time feedback from the search tool.
This work emphasizes the importance of multi-turn interactions in bridging the gap between internal textual reasoning and external information retrieval, demonstrating improved reasoning robustness and answer quality over traditional single-turn or passive retrieval approaches.  Despite these advances, there remains a lack of algorithms specifically tailored to stabilize PPO in multi-turn LLM off-policy training. Our work addresses this gap by introducing a customized PPO variant that explicitly leverages turn-level structure, leading to more stable and scalable training.

\section{Preliminaries} \label{sec:preliminary}
\textbf{Notation.} An autoregressive language model parameterized by $\theta$ is defined as a policy~$\pi_{\theta}$. We use $x$ to denote a query and $\mathcal{D}$ to denote the query set. Given a response $y$ to a query $x$, its likelihood under the policy $\pi_\theta$ is denoted as $\pi_\theta(y \mid x) = \prod_{t=1}^{\lvert y\rvert} \pi_\theta(y_t \mid x, y_{<t}),$ where $\lvert y\rvert$ denotes the number of tokens in~$y$. A query-response pair $(x,y)$ can be scored by a verifier $r$, resulting in a reward $r(x,y)\in[0,1]$.

\textbf{Turn-level MDP.} We model multi-turn interactions as a Markov decision process $\mathcal{M}=(\mathcal{S},\mathcal{A},\mathcal{P},\mathcal{R},\gamma)$ defined at the granularity of turns. A turn is denoted by $y^k \coloneqq (y_{t_k^\text{start}}, \ldots, y_{t_k^\text{end}})$, and a full trajectory can be written as $y \coloneqq (y^1; \ldots; y^n)$ for $n$ turns, where `$;$' denotes token concatenation. Each state $s_k \in \mathcal{S}$ corresponds to the dialogue context at the beginning of turn $k$, which we define as $s_k \coloneqq (x, y^1, y^2, \ldots, y^{k-1})$. The transition $\mathcal{P}(s_{k+1}\mid s_k,a_k)$ updates the context with environment feedback (e.g., results from a tool call), and the reward $\mathcal{R}(s_k,a_k)$ evaluates the quality of the turn. The discount factor $\gamma \in (0,1]$ balances present and future rewards. Unlike token-level MDPs, the unit of interaction here is the \emph{turn}, with boundaries $[t_k^\text{start}, t_k^\text{end}]$, aligning credit assignment with the natural structure of reasoning and retrieval phases.

\textbf{Proximal Policy Optimization (PPO).}
Using samples generated from the behavior (old) policy $\pi_{\theta_{\mathrm{old}}}$, PPO constrains the policy update within a proximal region of the old policy through the clipping mechanism. This allows for multiple gradient updates to the policy using the same batch of samples. Specifically, PPO employs the following objective for policy optimization:
\begin{align}
    &\mathcal{J}_{\mathrm{PPO}}(\theta) =
    \mathbb{E}_{x \sim \mathcal{D},\, y \sim \pi_{\theta_{\mathrm{old}}}(\cdot|x)}
    \Biggl[ \frac{1}{|y|} \sum_{t=1}^{|y|}
    \min \left( w_t(\theta) \hat{A}_t,\
    \mathrm{clip} \left( w_t(\theta), 1 \!-\! \epsilon, 1 \!+\! \epsilon \right) \hat{A}_t \right) \Biggr],
    \label{eq:ppo_obj}
\end{align}
where the importance ratio of the token $y_t$ is defined as
$
w_t(\theta) =
\frac{\pi_{\theta}(y_t \,|\, x, y_{<t})}
     {\pi_{\theta_{\mathrm{old}}}(y_t \,|\, x, y_{<t})},
$
and $\hat{A}_t$ denotes the token-level advantage estimated by GAE based on the critic model $\hat{V}$.
As discussed in related work, PPO has been widely used in many agentic tasks and performs well in single-turn settings.
However, its performance in multi-turn settings is often unstable and prone to collapse; for example, \citet{yuan2025s} shows that PPO can fail in long reasoning tasks. In multi-turn settings, reward signals are typically provided only at the end of a long interaction.
This delayed feedback leaves value predictions at intermediate states poorly constrained, and the resulting errors propagate backward through temporal-difference updates, compounding over many steps \citep{arjona2019rudder}. Another possible reason for PPO collapse is the high variance induced by importance weighting. In long-horizon and off-policy settings, the importance ratios can become highly variable due to aggressive updates and factors such as training–inference mismatch, which may lead to unstable updates even with clipping. Similar clipping mechanisms are also used in other RL methods, such as GRPO. We observe that once the variance associated with importance weighting becomes large, the training is severely destabilized.
\begin{figure*}[!t]
    \centering
    \begin{subfigure}{0.24\linewidth}
        \centering
        \includegraphics[width=\linewidth]{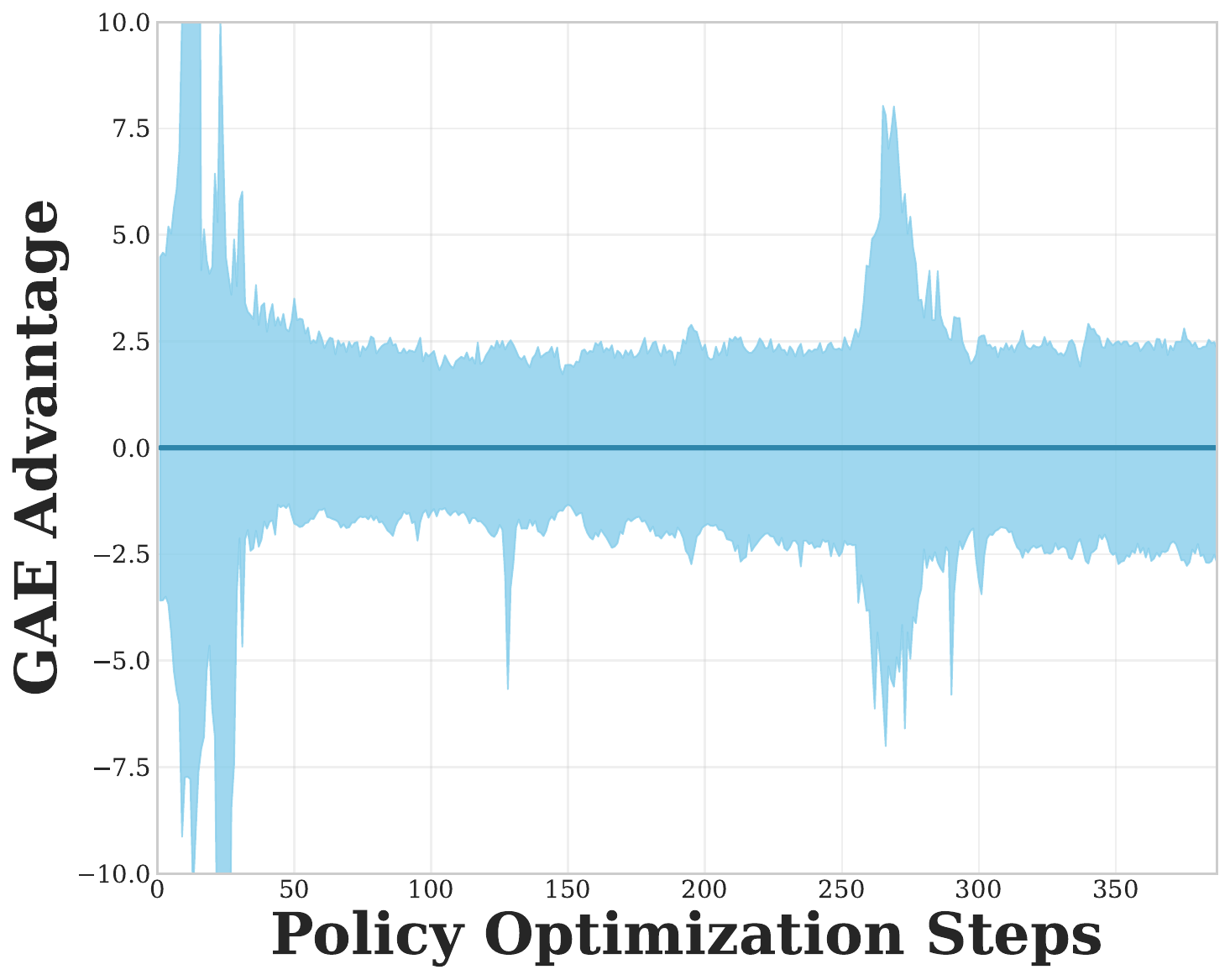}
        \caption{Advantage.}
        \label{fig:gae_adv}
    \end{subfigure}
    \begin{subfigure}{0.24\linewidth}
        \centering
        \includegraphics[width=\linewidth]{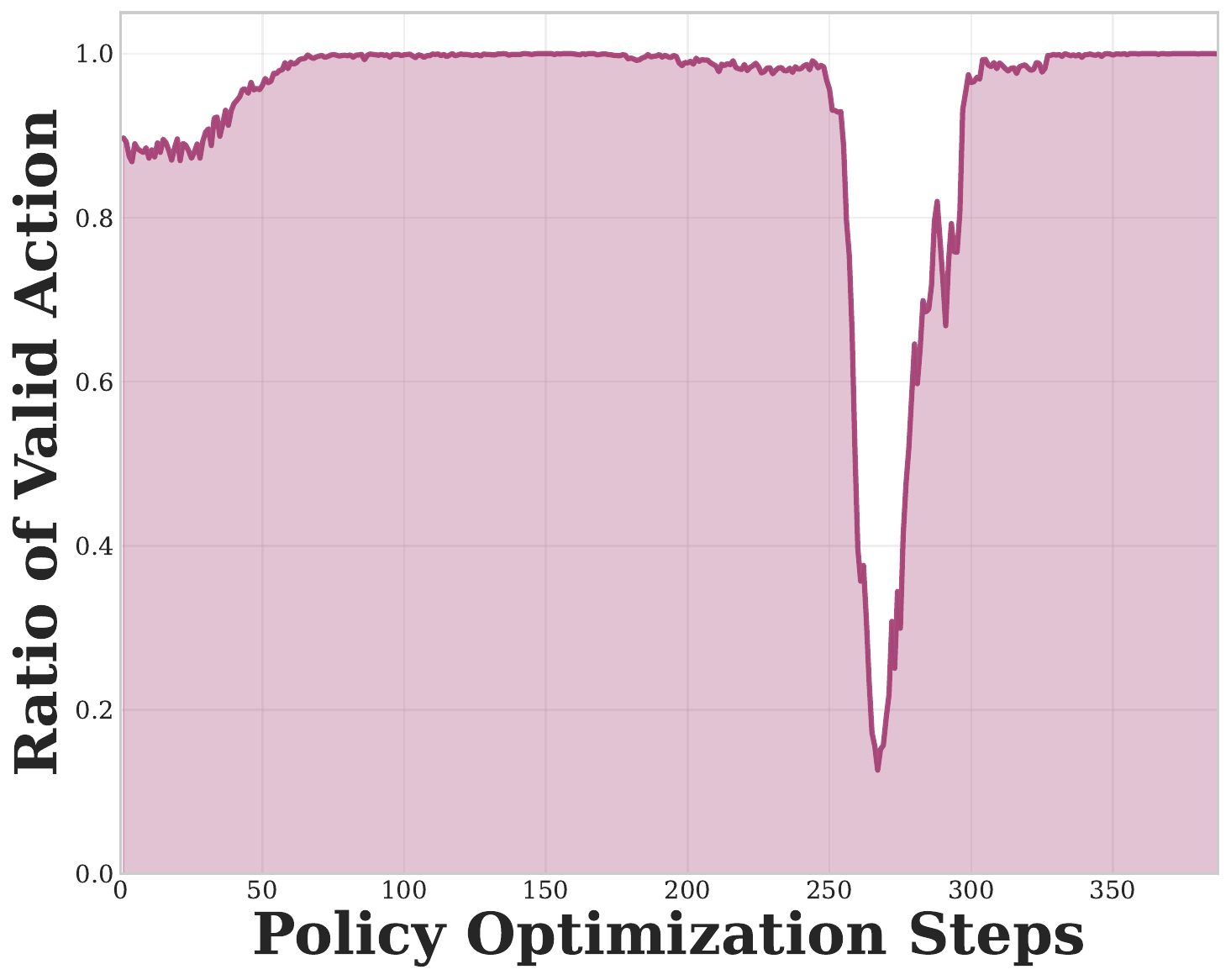}
        \caption{Ratio of valid actions.}
        \label{fig:valid_ratio}
    \end{subfigure}
    \begin{subfigure}{0.24\linewidth}
        \centering
        \includegraphics[width=\linewidth]{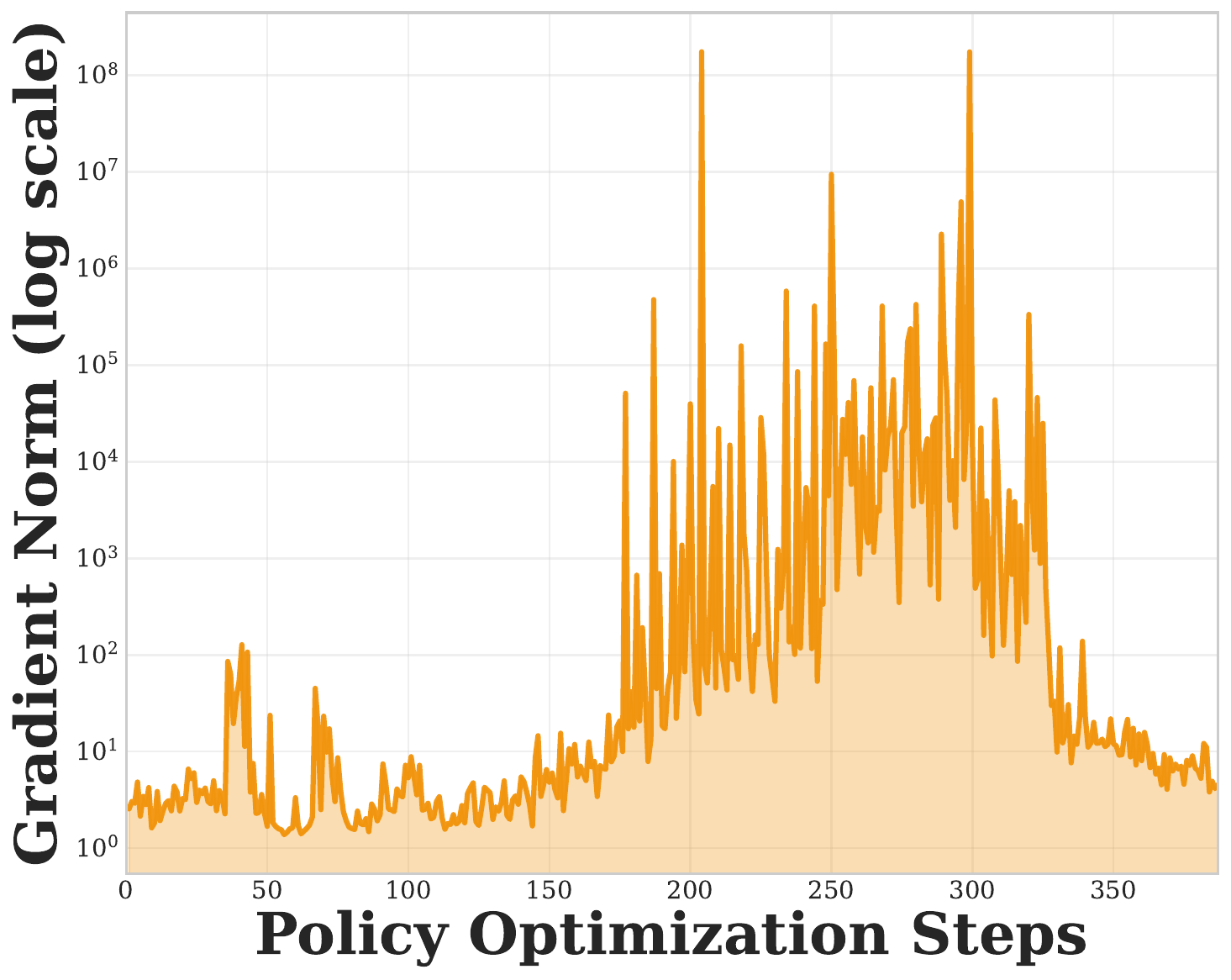}
        \caption{Policy gradient.}
        \label{fig:grad_norm}
    \end{subfigure}
    \begin{subfigure}{0.24\linewidth}
        \centering
        \includegraphics[width=\linewidth]{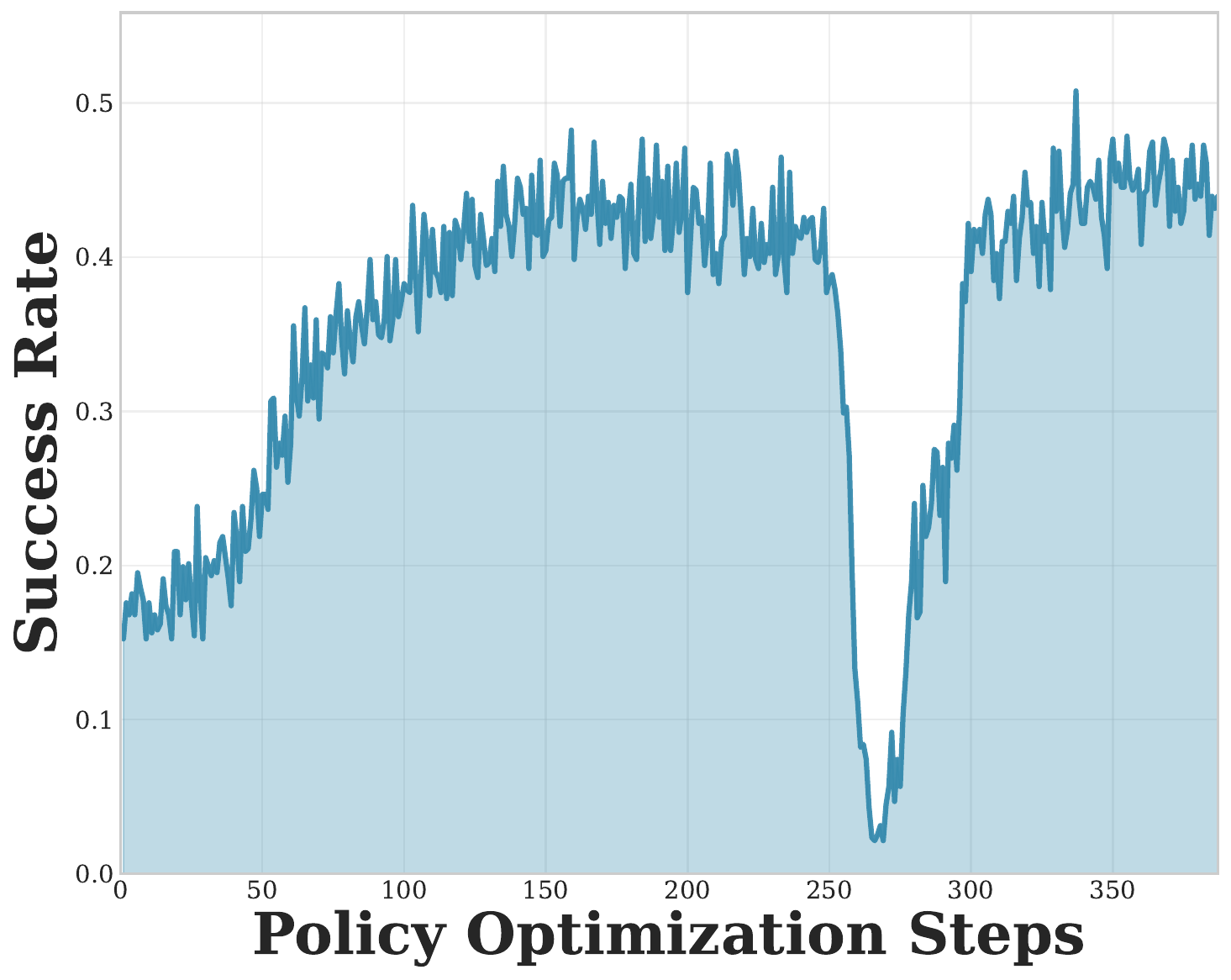}
        \caption{Success rate.}
        \label{fig:success_rate}
    \end{subfigure}
    \caption{Observations from a failed training run using the Qwen2.5-7B model with token-level PPO. From left to right, we report the estimated advantage (batch mean and variance), the ratio of valid actions (i.e., successful tool invocations), the $\ell_2$ norm of the policy gradient, and the search-task success rate. All metrics are logged per training batch.}
    \label{fig:fail_observation}
\end{figure*}

\textbf{Verification.} We provide both metric diagnostics and a representative failure case by examining a failed run of Qwen2.5-7B under token-level PPO (Figure~\ref{fig:fail_observation}) in multi-turn search tasks. Specifically, we record the numerical metrics during training, and the results reveal two critical issues: (1) around steps 250--300, the critic assigns highly variable advantages while the ratio of valid actions sharply collapses, indicating severe tool-use degradation; and (2) these unreliable samples subsequently trigger an explosion in the policy gradient norm, rapidly degrading performance and leading to a collapse in task success rate. {This evidence confirms that critic estimation errors and variance accumulation can destabilize off-policy updates in multi-turn settings, thereby motivating the need for stabilization mechanisms beyond turn-level importance sampling alone.} To further examine this behavior, we provide textual demonstrations in Appendix~\ref{app:fail_example_text}, which show that prior to the collapse (at policy optimization steps 243 and 247), importance sampling on format-related tokens exhibits high variance.
Although these tokens are not critical to the LLM’s reasoning ability, such variance can accumulate and lead to format collapse: after policy optimization step 250, the LLM increasingly fails to follow the predefined output format, consistent with the behavior illustrated in Figure~\ref{fig:valid_ratio}.
\section{Proposed Algorithm}\label{sec:algo}
In this section, we present our approach for stabilizing reinforcement learning in multi-turn LLM training, termed \aname~(\underline{S}tabilizing \underline{O}ff-Policy \underline{R}einforcement \underline{L}earning for Long-Horizon Agent Training). Before presenting the details, we first introduce two key components of our framework and analyze how they stabilize training from the perspective of gradient norms.
\subsection{Turn-Level Importance Sampling} \label{sec:turn_level_ppo}
In the previous section, we provided a representative failure case from a failed training run (see Appendix~\ref{app:fail_example_text}), showing that token-level credit assignment in PPO can destabilize training and impair tool-calling stability.
Building on this, we introduce the first component of our framework, which provides a simple and effective way to adjust the granularity of credit assignment.
Specifically, for $K$ total turns, with $[t_k^{\text{start}}, t_k^{\text{end}}]$ denoting the token positional boundaries of turn $k \in \{1, \dots, K\}$, we formulate a turn-level PPO objective:

\begin{align}
    &\mathcal{J}_{\mathrm{Turn-PPO}}(\theta)
    = \mathbb{E}_{x \sim \mathcal{D}, y \sim \pi_{\theta_{\mathrm{old}}}(\cdot|x)} \!\Biggl[\frac{1}{|y|}\sum_{k=1}^{K} \sum_{t=t_k^{\text{start}}}^{t_k^{\text{end}}} \!\min\!\left\{ w_k^{\text{turn}}(\theta) \hat{A}_{t}, \text{clip}(w_k^{\text{turn}}(\theta), 1\!-\!\epsilon, 1\!+\!\epsilon) \hat{A}_{t} \right\} \!\Biggr],
    \label{eq:turn_level_PPO}
\end{align}
where $w_k^\text{turn}$ is a turn-level importance ratio. Inspired by the sequence-level weight used by \citet{zheng2025group}, we define a turn-level variant as:
\begin{align}\label{eq:turn_level_IS}
w_k^{\text{turn}}(\theta) &\coloneqq \left( \frac{\pi_\theta(y^k \mid x, y^{<k})}{\pi_{\theta_{\mathrm{old}}}(y^k \mid x, y^{<k})} \right)^{\frac{1}{|y^k|}} = \exp\left( \frac{1}{|y^k|} \sum_{t=t_k^\text{start}}^{t_k^\text{end}} \log \frac{\pi_\theta(y_{t} \mid x, y_{<t})}{\pi_{\theta_{\mathrm{old}}}(y_{t} \mid x, y_{<t})} \right),
\end{align}
where the ratio is scaled by the length of the $k^\text{th}$ turn, $\lvert y^k \rvert$, to stabilize the objective across turns of varying lengths. {Next, we introduce a lemma showing that turn-level importance sampling induces a turn-wise aggregation of token-level advantages, such that all tokens within the same turn share a common, length-normalized credit signal. As a result, policy gradients are computed at the granularity of turns rather than individual tokens, aligning credit assignment with turn-level (action-level) interactions.} To facilitate this analysis, we define the events under which clipping is inactive for the objective in Eq.~\ref{eq:turn_level_PPO}:
\begin{align}
\mathcal{B}_{\text{turn}}^+ \coloneqq
    \{\, (k,t) : \hat{A}_t \ge 0,\; w_k^{\text{turn}}(\theta) \le 1 + \epsilon \,\}, ~~~\mathcal{B}_{\text{turn}}^- \coloneqq
    \{\, (k,t) : \hat{A}_t < 0,\; w_k^{\text{turn}}(\theta) \ge 1 - \epsilon \,\}.
\end{align}
Let $\mathcal{B}_{\text{turn}} \coloneqq \mathcal{B}_{\text{turn}}^+ \cup \mathcal{B}_{\text{turn}}^-$. For token positions $(k,t)\in\mathcal{B}_{\text{turn}}$, the PPO objective reduces to the unclipped form $w_k^{\text{turn}}(\theta)\hat{A}_t$. In contrast, for token positions $(k,t)\in\mathcal{B}_{\text{turn}}^c$, where clipping is active, the clipping indicates a mismatch between the behavior (old) and current policies, and the gradients of the corresponding tokens are set to zero.

\begin{lemma}
\label{lm:proposed_ppo_grad} The gradient of the objective function in Eq.~\ref{eq:turn_level_PPO} is given by
    {\small \begin{align}
        \nabla_{\theta} \mathcal{J}_{\mathrm{Turn-PPO}}(\theta)
        =\mathbb{E} \bigg[\frac{1}{|y|}
        \sum_{k=1}^{K}
        \underbrace{
        w_k^{\text{turn}}(\theta)\frac{\hat{A}^k}{|y^k|} }_{\text{Turn-Level Credit}}
        \nabla_{\theta} \log \pi_\theta(y^k | x, y^{<k}) \bigg],\label{eq:turn_level_credit_assignment}
    \end{align}}
    where $|y^k| = t_k^{\text{end}} - t_k^{\text{start}} + 1$ and $\hat{A}^k \coloneqq \sum_{t=t_k^{\text{start}}}^{t_k^{\text{end}}} \mathds{1}_{\{(k,t) \in \mathcal{B}_{\text{turn}} \}} \ \hat{A}_{t}$.
\end{lemma}
\begin{proof}
    We refer to Appendix~\ref{app:proposed_ppo_grad} for the proof.
\end{proof}
Lemma~\ref{lm:proposed_ppo_grad} highlights the effect of turn-level importance sampling. As shown in Eq.~\ref{eq:turn_level_credit_assignment}, all tokens within a turn share the same aggregated advantage $\hat{A}^k$, while advantages and clipping are still computed token-wise. This implies that the policy gradient assigns credit at the granularity of turns rather than individual tokens, specifically due to the use of a shared turn-level importance weight. As a result, tokens within the same turn collectively reinforce or suppress the corresponding turn-level action, instead of competing for credit at the token level. Importantly, this modification does not alter how advantages are estimated or how clipping is applied. Instead, it reshapes how token-level signals are aggregated in the gradient, mitigating large gradient norms induced by token-level fluctuations while preserving fine-grained advantage estimation. This turn-level structure arises from using a shared geometric-mean importance ratio within each turn and aggregating unclipped token-level contributions in $\mathcal{B}_{\text{turn}}$. {This turn-wise aggregation is consistent with prior work that performs credit assignment at the granularity of interaction turns or hierarchical agent actions~\citep{zhou2024archer,zeng2025reinforcing}.}

\begin{figure*}[!t]
    \centering
    \begin{subfigure}{0.24\linewidth}
        \centering
        \includegraphics[width=\linewidth]{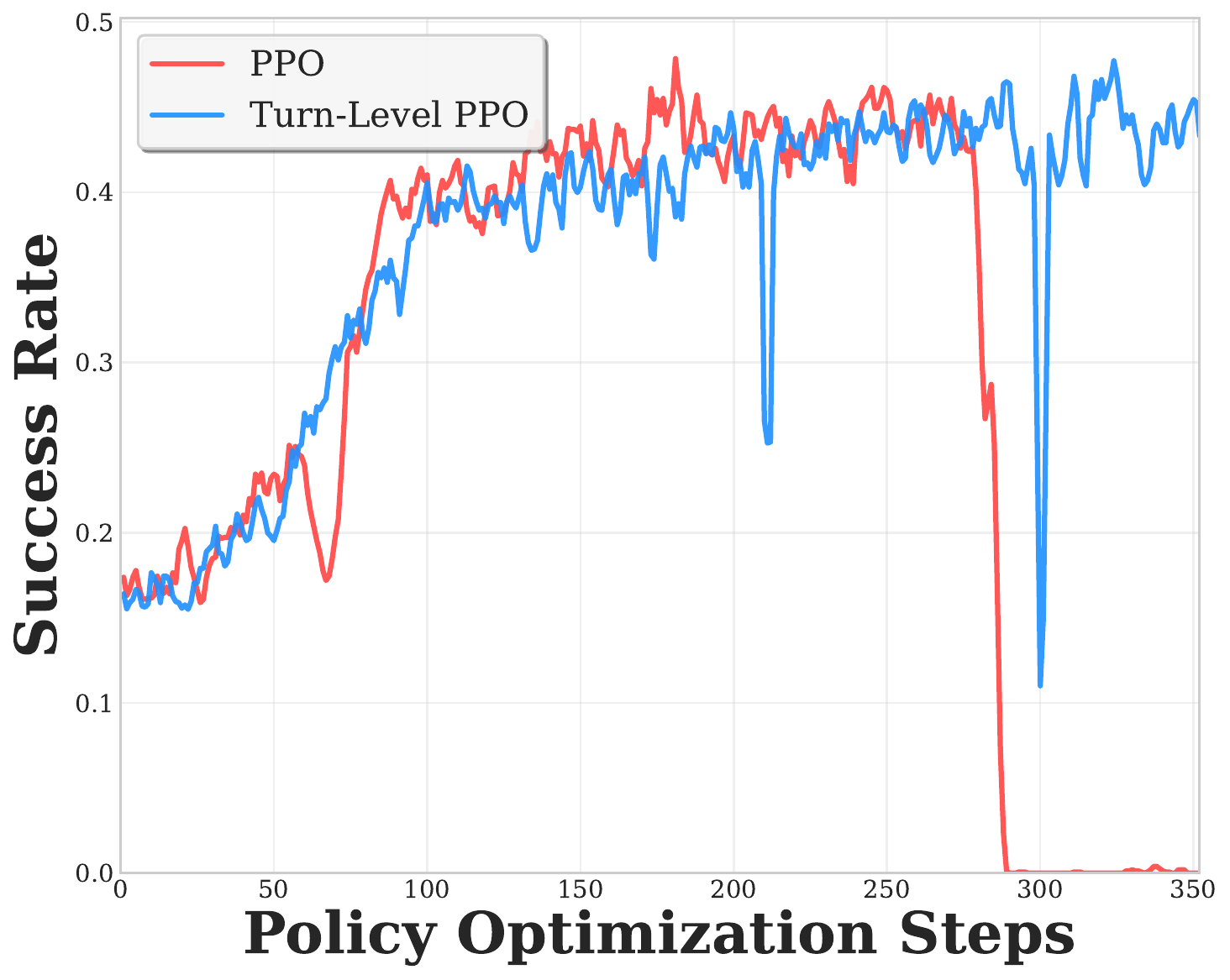}
        \caption{Success rate.}
        \label{fig:performance_curve_token_turn}
    \end{subfigure}
    \begin{subfigure}{0.24\linewidth}
        \centering
        \includegraphics[width=\linewidth]{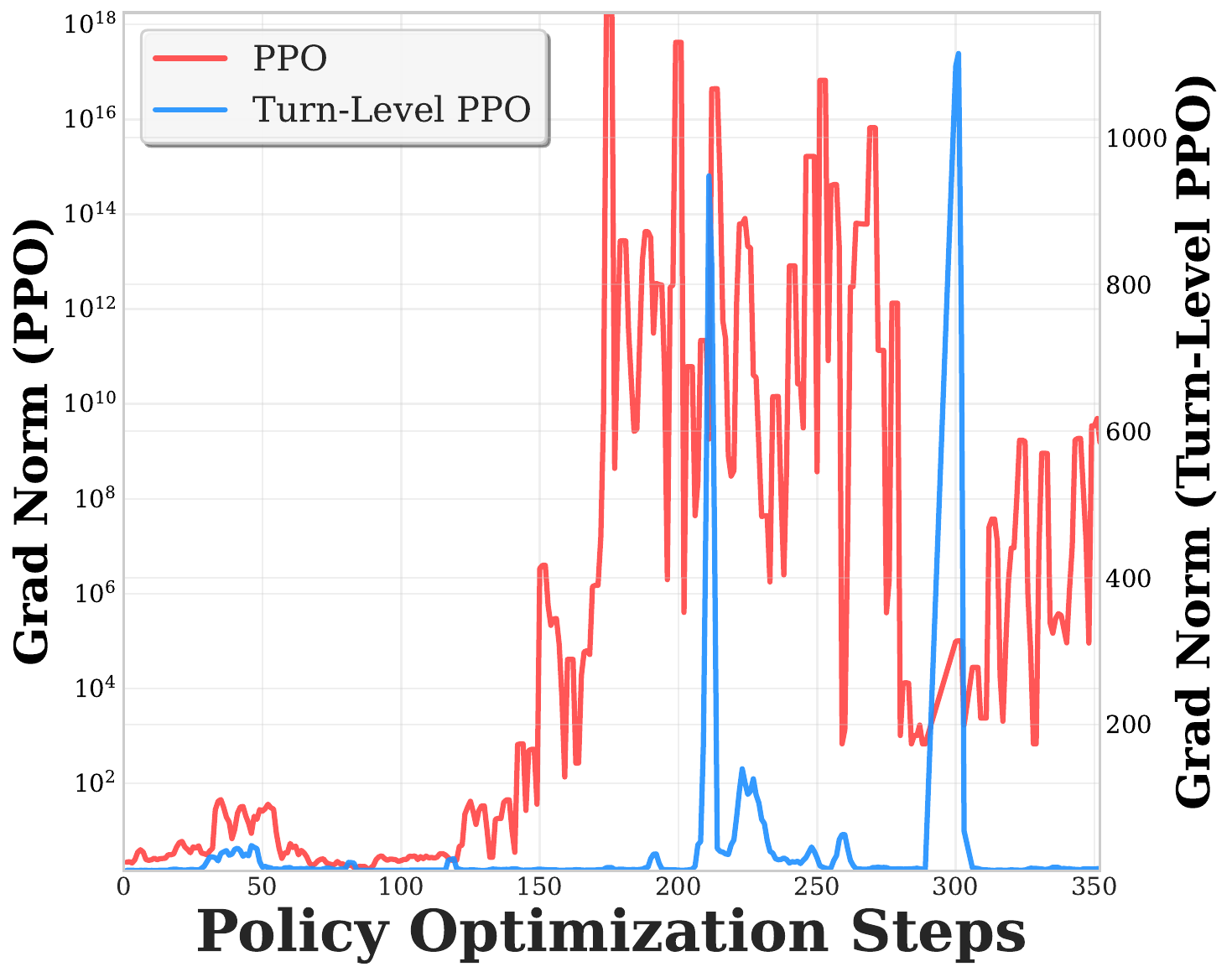}
        \caption{Gradient norm.}
        \label{fig:gradient_norm_token_turn}
    \end{subfigure}
    \begin{subfigure}{0.24\linewidth}
        \centering
        \includegraphics[width=\linewidth]{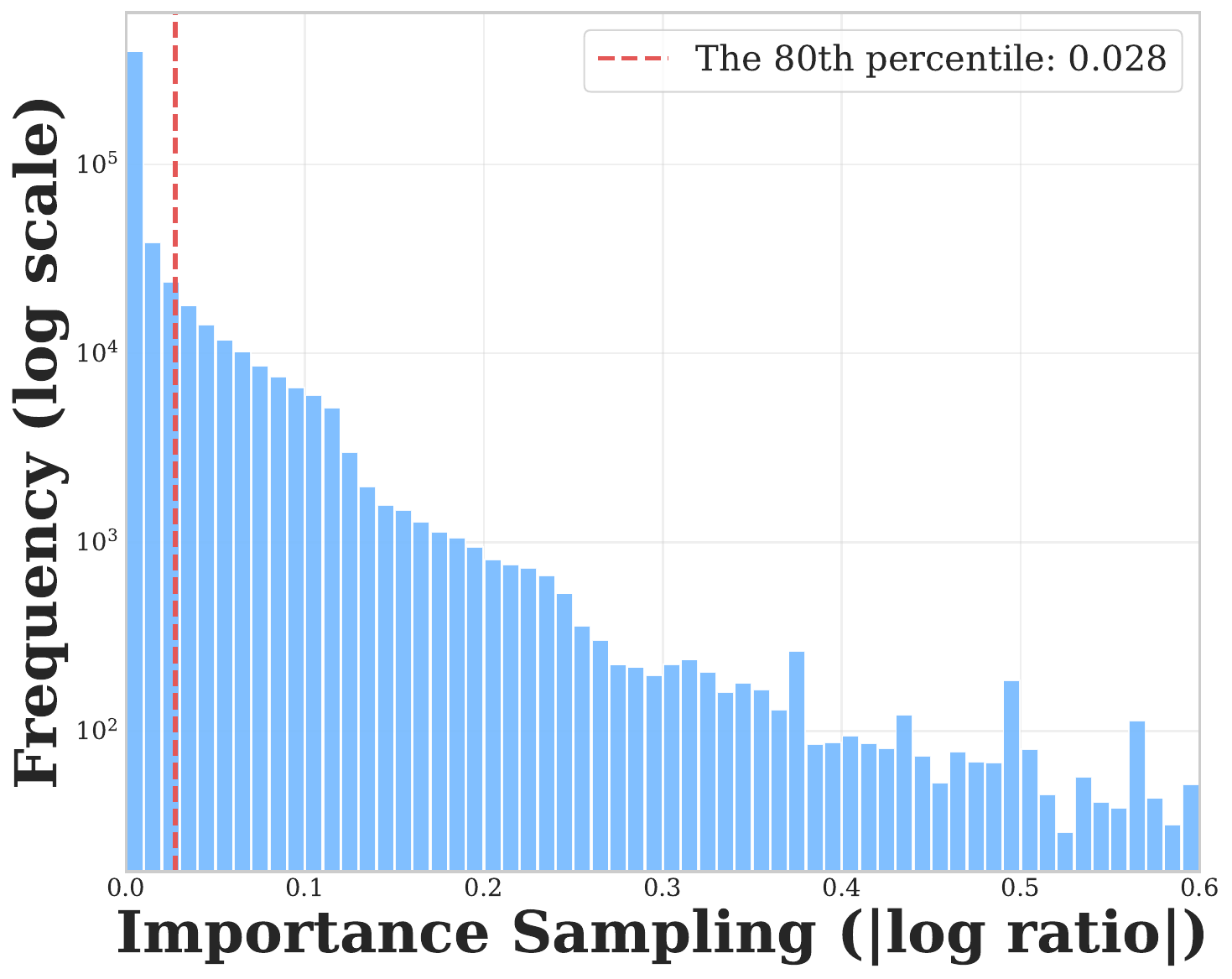}
        \caption{IS distribution.}
        \label{fig:IS_distribution}
    \end{subfigure}
    \begin{subfigure}{0.24\linewidth}
        \centering
        \includegraphics[width=\linewidth]{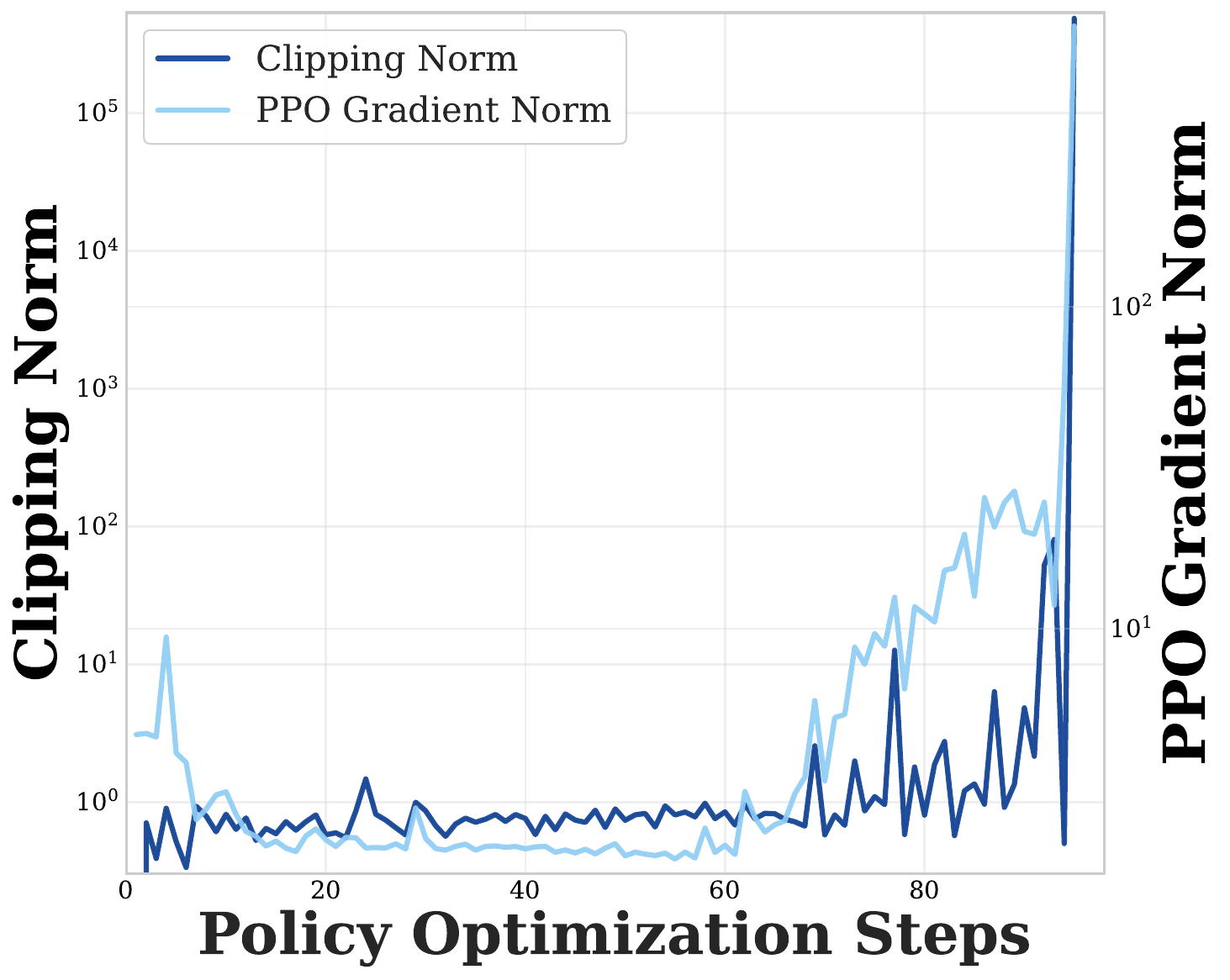}
        \caption{Clipping-bias norm.}
        \label{fig:Clipping_norm}
    \end{subfigure}
    \caption{Comparison and diagnostics on Qwen2.5-7B base model for the search task. Left: Token- vs.\ turn-level PPO, showing improved stability in (a) success rate and (b) policy gradient norm. Right: Clipping-bias diagnostics, including (c) heavy-tailed importance ratios (log scale) and (d) the growth of the clipping-bias norm alongside the gradient norm, indicating increasing off-policy instability. All metrics are logged per training batch.}
    \label{fig:token_turn}
\end{figure*}

\textbf{Validation.} We compare token-level PPO in Eq.~\ref{eq:ppo_obj} and turn-level PPO in Eq.~\ref{eq:turn_level_PPO} on the search task using the Qwen2.5-7B base model. {As shown in Figure~\ref{fig:gradient_norm_token_turn}, turn-level importance sampling reduces the norm of the PPO policy gradient by several orders of magnitude.} This reduction effectively prevents the catastrophic gradient explosion observed in standard token-level PPO and leads to more stable optimization dynamics during training. However, we observe that the performance curve of turn-level PPO still exhibits occasional sharp drops. These drops coincide with spikes in the gradient norm (approximately 1000--2000), indicating that instability can still arise even under turn-level importance sampling. {As a result, turn-level PPO fails to sustain its best performance.}
\subsection{Clipping-Triggered Normalization}
\label{sec:stable_train_clipping_bias}

While turn-level importance sampling alleviates gradient explosion by aligning credit assignment with the turn structure, the remaining instability suggests that granularity mismatch alone does not fully explain the observed optimization failures.
In highly off-policy regimes, the reliability of sampled data itself becomes a critical factor affecting training stability, independent of how credit is assigned.
{To better understand this effect, we begin with a general analysis of off-policy PPO updates by examining the gradient decomposition. This analysis motivates a normalization mechanism that targets unreliable off-policy updates and can be naturally integrated into clipping-based reinforcement learning algorithms.
In the following, we apply this mechanism to the turn-level PPO developed in Section~\ref{sec:turn_level_ppo}, yielding a more robust and stable optimization procedure.}

We now investigate the growth of the policy gradient norm during training by examining the distribution of importance ratios at the mid-training stage. As shown in Figure~\ref{fig:IS_distribution}, the distribution is heavy-tailed, with some samples exhibiting extremely large or small importance ratios, a phenomenon that typically arises when the training batch becomes highly off-policy. To mitigate this effect and avoid updates dominated by batches with high-variance importance ratios, we introduce the second component of our framework, termed Clipping-Triggered Normalization. To motivate this design, we analyze the gradient of the  PPO objective by decomposing it according to whether clipping is active. To define the clipping bias, we begin by defining the events under which clipping is inactive:
\[
\mathcal{B}^+ \coloneqq
    \{\, t : \hat{A}_t \ge 0,\; w_t(\theta) \le 1 + \epsilon \,\}, \quad   \mathcal{B}^- \coloneqq
    \{\, t : \hat{A}_t < 0,\; w_t(\theta) \ge 1 - \epsilon \,\}.
\]
These sets correspond to token positions where the PPO surrogate objective reduces to the unclipped importance-weighted objective. We define $\mathcal{B} \coloneqq \mathcal{B}^+ \cup \mathcal{B}^-$. The complement set $\mathcal{B}^c$ consists of all remaining token positions where clipping is active. Then we can formally decompose the PPO gradient as follows:
\begin{lemma}
    The PPO gradient can be decomposed as  {\begin{align}\label{eq:ppo_composition}
        \nabla_\theta \mathcal{J}_{\mathrm{PPO}}(\theta) = &
        \underbrace{\mathbb{E}\bigg[\frac{1}{|y|} \!\sum_{t=1}^{|y|} w_t \nabla_\theta \log \pi_\theta \hat{A}_t\bigg]}_{\text{Unclipped Off-Policy Gradient}} -   \underbrace{\mathbb{E}\!\bigg[\frac{1}{|y|} \sum_{t=1}^{|y|} \mathds{1}_{t \in \mathcal{B}^c} \, w_t\, \nabla_\theta \log \pi_\theta\, \hat{A}_t \bigg]}_{\text{$C(\theta)$: Clipping-Bias Vector}},
    \end{align}}
where $\hat{A}_t$ denotes the advantage estimated by the critic using the GAE method.
\label{lm:ppo_grad_decomp}
\end{lemma}
\begin{proof}
    We refer to Appendix~\ref{app:ppo_grad_decomp} for the proof of Lemma~\ref{lm:ppo_grad_decomp}.
\end{proof}
From the gradient decomposition in Lemma~\ref{lm:ppo_grad_decomp}, the PPO gradient can be expressed as the difference between an unclipped off-policy gradient and a clipping-bias vector $C(\theta)$.
The former corresponds to the vanilla off-policy policy gradient without clipping, where all tokens contribute according to their importance-weighted advantages, while the latter arises from the indicator $\mathds{1}_{t \in \mathcal{B}^{c}}$ and captures the gradient contribution of tokens for which clipping is active. The clipping mechanism stabilizes optimization by suppressing excessively large updates, and the resulting clipping-bias vector represents the gradient contribution removed by clipping, thereby introducing bias into the gradient estimate.

\textbf{Validation.}
We study how the clipping-bias norm evolves during training and how it relates to the PPO gradient by running token-level PPO on the \mbox{Qwen2.5-7B} base model.
Specifically, we track the $\ell_2$ norms of the clipping-bias vector $C(\theta)$ and the PPO gradient $\nabla \mathcal{J}_{\mathrm{PPO}}$ over optimization steps (Figure~\ref{fig:Clipping_norm}).
We find that the clipping-bias norm and the PPO gradient norm are clearly correlated and tend to increase together as training proceeds.
Moreover, the clipping-bias norm shows larger fluctuations in later stages of training.
A large clipping-bias norm indicates a large aggregate gradient contribution from clipped tokens, suggesting that the current batch is strongly off-policy.
As a result, the policy updates become less stable later in training, when off-policy effects are more pronounced.

{Motivated by this observation, we incorporate a new stabilization principle into the turn-level PPO formulation (Section~\ref{sec:turn_level_ppo}).}
Specifically, we propose a surrogate gradient estimator that rescales each mini-batch update according to its estimated clipping-bias norm, thereby adaptively downweighting highly off-policy updates and suppressing destabilizing gradients.
From a scaling perspective, turn-level importance sampling treats each turn as a single optimization unit, weighting it by the turn-level importance ratio and the cumulative unclipped advantage, normalized by turn length.
This results in a surrogate gradient update for turn-level PPO with an explicit clipping-bias normalization, characterized by the following clipping-bias vector:
\[
{
    C_{\text{turn}}(\theta)
    :=
    \mathbb{E}_{\tau \sim \pi_{\theta_{\mathrm{old}}}}
    \!\left[
    \frac{1}{|y|}
    \sum_{k=1}^{K}
    \sum_{t=t_k^{\text{start}}}^{t_k^{\text{end}}}
    \mathds{1}_{(k,t) \in \mathcal{B}_{\text{turn}}^{c}}
    \, w_k^{\text{turn}}(\theta)
    \, \hat{A}_t
    \, \nabla_\theta \log \pi_\theta(y_t\mid x,y_{<t})
    \right].
}
\]

Intuitively, the  {clipping-bias norm} reflects the degree to which a batch is affected by off-policy samples: a larger clipping-bias norm indicates a  {larger aggregate gradient contribution from clipped tokens}, which is associated with more severe policy mismatch. Rescaling the policy update by the inverse of this norm therefore adaptively reduces the effective step size when off-policy variance dominates.

Based on this idea, we define the stabilized policy gradient as
\begin{align}
    \nabla_{\theta} \mathcal{J}_{\mathrm{SO\text{-}PPO}}(\theta)
    \coloneqq
    \frac{1}{\|C_{\text{turn}}(\theta)\|_2}
    \nabla_{\theta} \mathcal{J}_{\mathrm{Turn\text{-}PPO}}(\theta).
    \label{eq:ST_PPO}
\end{align}

In expectation, both the turn-level PPO gradient and the  {clipping-bias vector} are defined over trajectories sampled from the same  {behavior (old) policy}:
\begin{align}
\nabla_{\theta}\mathcal{J}_{\mathrm{Turn\text{-}PPO}}(\theta)
&:=
\mathbb{E}_{\tau \sim  {\pi_{\theta_{\mathrm{old}}}}}
\!\left[
\nabla_{\theta}\mathcal{J}_{\mathrm{Turn\text{-}PPO}}(\theta;\tau)
\right], \\
C_{\text{turn}}(\theta)
&:=
\mathbb{E}_{\tau \sim  {\pi_{\theta_{\mathrm{old}}}}}
\!\left[
C_{\text{turn}}(\theta;\tau)
\right].
\end{align}

In practice, these expectations are approximated using the same mini-batch of sampled trajectories, yielding the following Monte Carlo estimator:
\begin{align}
    \nabla_{\theta} \hat{\mathcal{J}}_{\mathrm{SO\text{-}PPO}}(\theta)
    \;:=\;
    \frac{1}{\max\!\left(\left\|\hat{C}_{\text{turn}}(\theta)\right\|_2, \delta\right)}
    \nabla_{\theta}\hat{\mathcal{J}}_{\mathrm{Turn\text{-}PPO}}(\theta),
    \label{eq:ST_PPO_mc}
\end{align}
where both $\hat{C}_{\text{turn}}(\theta)$ and
$\nabla_{\theta}\hat{\mathcal{J}}_{\mathrm{Turn\text{-}PPO}}(\theta)$
are computed from the same set of trajectories
$\{\tau^{(i)}\}_{i=1}^{B}$ sampled from the  {behavior (old) policy}
 {$\tau^{(i)} \sim \pi_{\theta_{\mathrm{old}}}$}. The constant $\delta>0$ imposes a lower bound on the normalization factor, preventing excessively large updates when the estimated  {clipping-bias norm} is small; we set $\delta = 1$ in all experiments. Figure~\ref{fig:framework_STPPO} illustrates the two components of our  {\pponame{} instantiation} and how they effectively reduce the norm of the policy gradient. {Algorithm~\ref{alg:ST_PPO}, provided in Appendix~\ref{app:ST_PPO_algorithm}, first computes the unnormalized turn-level PPO gradient in Eq.~\ref{eq:turn_level_credit_assignment} and then applies the clipping-triggered normalization in the practical update of Eq.~\ref{eq:ST_PPO_mc}.}
\subsection{Extension to GRPO}
We further demonstrate that the proposed stabilization mechanism extends beyond PPO by instantiating it in GRPO. We incorporate clipping-triggered normalization into GRPO, yielding a stabilized variant (\grponame), while retaining its  {group-relative advantage assignment at the trajectory level} without introducing turn-level importance sampling.

Formally, GRPO optimizes:
\begin{equation}
\mathcal{J}_{\mathrm{GRPO}}(\theta)
=
\mathbb{E}\Big[
\min\big(
w(\theta)\hat{A}^{\mathrm{grp}},
\operatorname{clip}(w(\theta), 1-\epsilon, 1+\epsilon)\hat{A}^{\mathrm{grp}}
\big)
\Big].
\end{equation}

We follow the standard GRPO formulation with token-level importance sampling and  {group-relative advantages} \citep{shao2024deepseekmath}. The corresponding {clipping-bias vector} is defined as:
\[
C_{\mathrm{GRPO}}(\theta)
:=
\mathbb{E}_{\tau\sim {\pi_{\theta_{\mathrm{old}}}}}\Big[
\frac{1}{|y|}
\sum_t
\mathds{1}_{t \in \mathcal{B}^c}
w_t(\theta)\hat{A}^{\mathrm{grp}}
    \nabla_\theta \log \pi_\theta(y_t\mid x,y_{<t})
\Big].
\]

The stabilized update is given by:
\[
\nabla_\theta \mathcal{J}_{\mathrm{SO\text{-}GRPO}}(\theta)
=
\frac{1}{\|C_{\mathrm{GRPO}}(\theta)\|_2}
\nabla_\theta \mathcal{J}_{\mathrm{GRPO}}(\theta).
\]

In practice, both terms are estimated via Monte Carlo using the same mini-batch; see Appendix \ref{app:SO_GRPO} for the full derivation.

This yields a lightweight yet effective variant: \grponame{} rescales the Monte Carlo policy gradient
by the inverse clipping-bias norm, thereby {downweighting mini-batch updates dominated by extreme off-policy samples}.
Since GRPO already uses a {group-relative advantage shared across all tokens within each trajectory}, we do not introduce additional turn-level importance sampling. We present \grponame{} in Algorithm~\ref{alg:SO_GRPO}.

\begin{figure*}[!t]
    \centering
    \begin{subfigure}{0.24\linewidth}
        \centering
        \includegraphics[width=\linewidth]{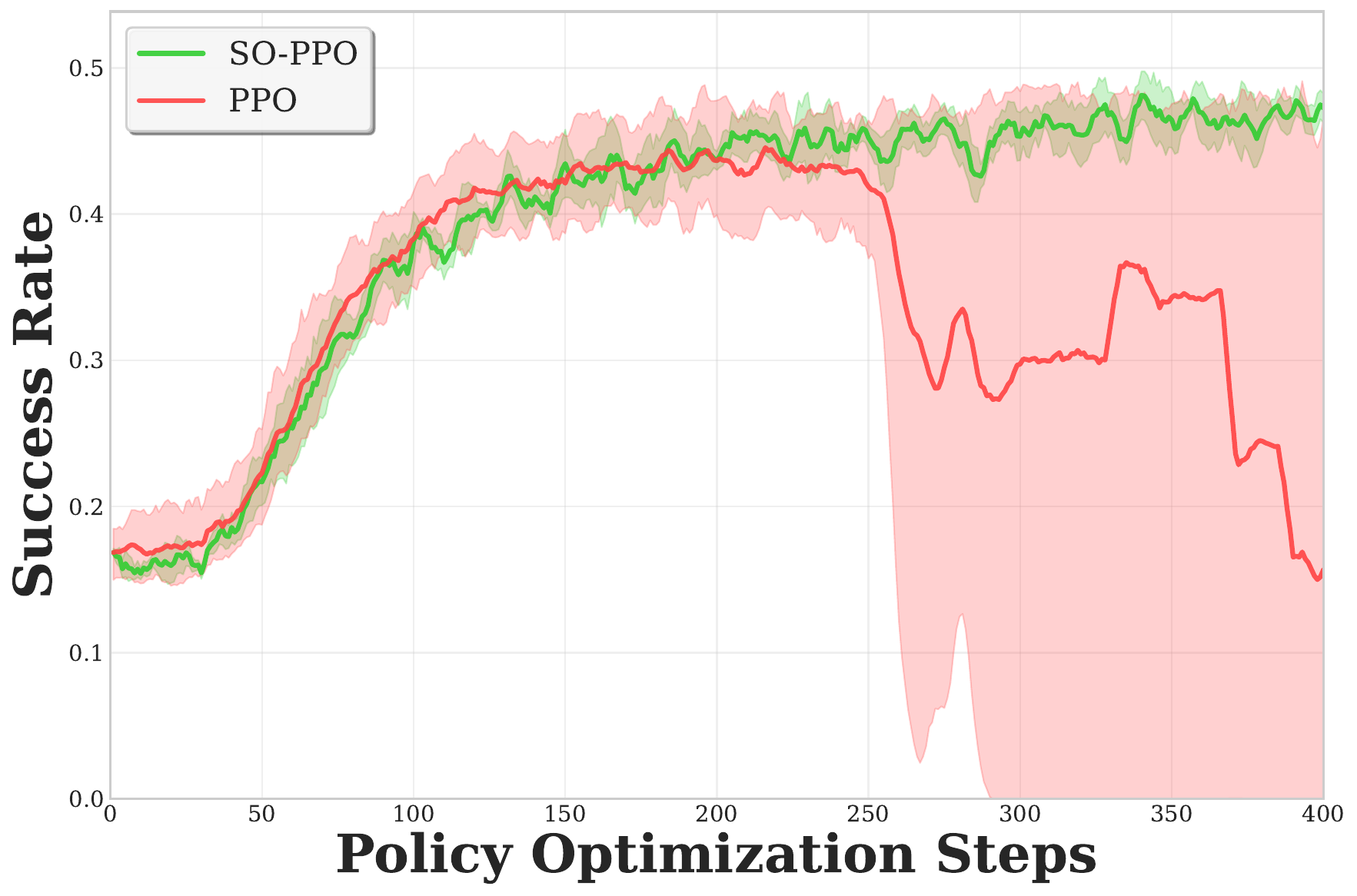}
        \caption{NQ dataset.}
        \label{fig:PPO_nq}
    \end{subfigure}
    \begin{subfigure}{0.24\linewidth}
        \centering
        \includegraphics[width=\linewidth]{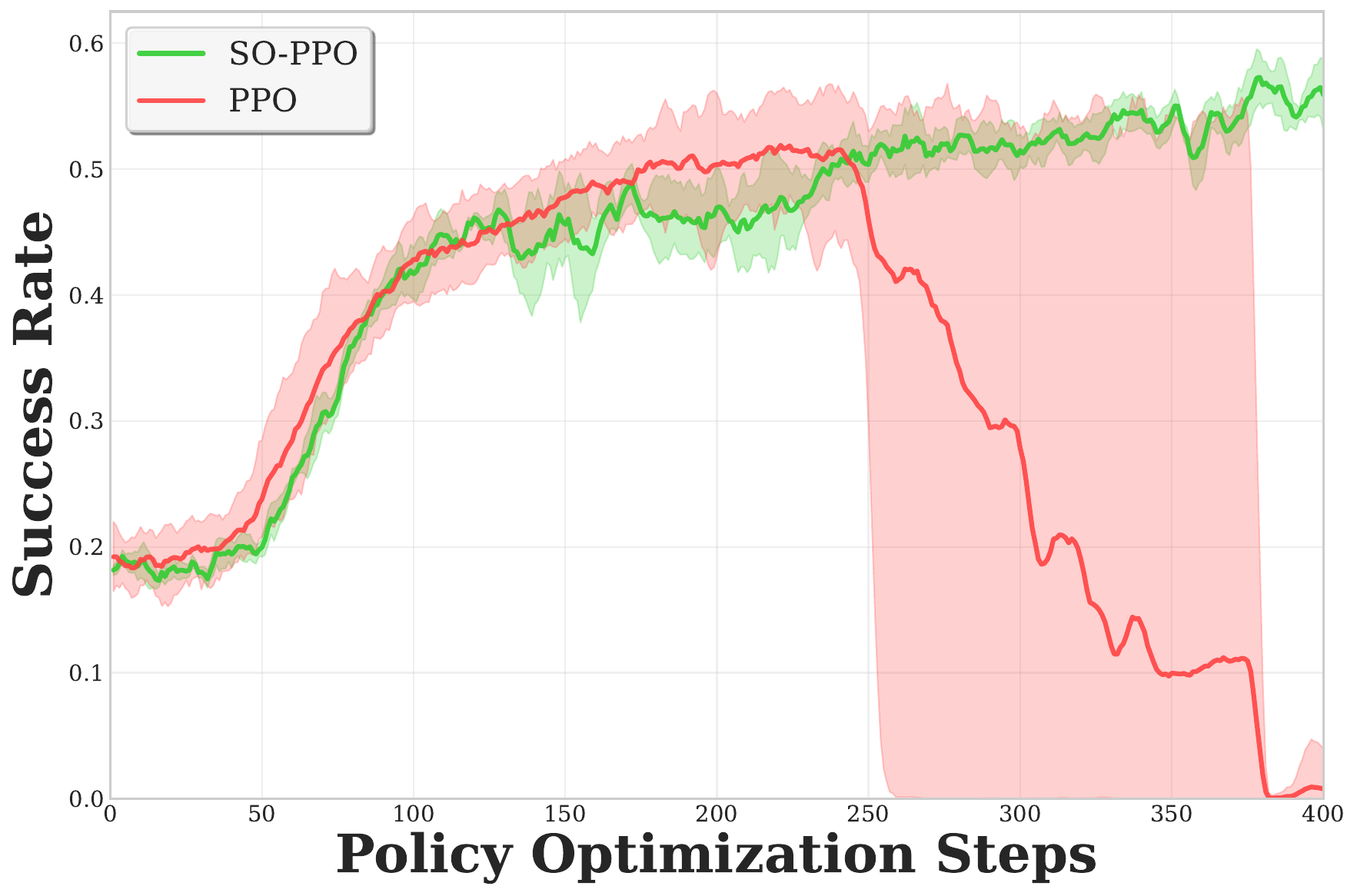}
        \caption{HotpotQA dataset.}
        \label{fig:PPO_hotpotqa}
    \end{subfigure}
    \begin{subfigure}{0.24\linewidth}
        \centering
        \includegraphics[width=\linewidth]{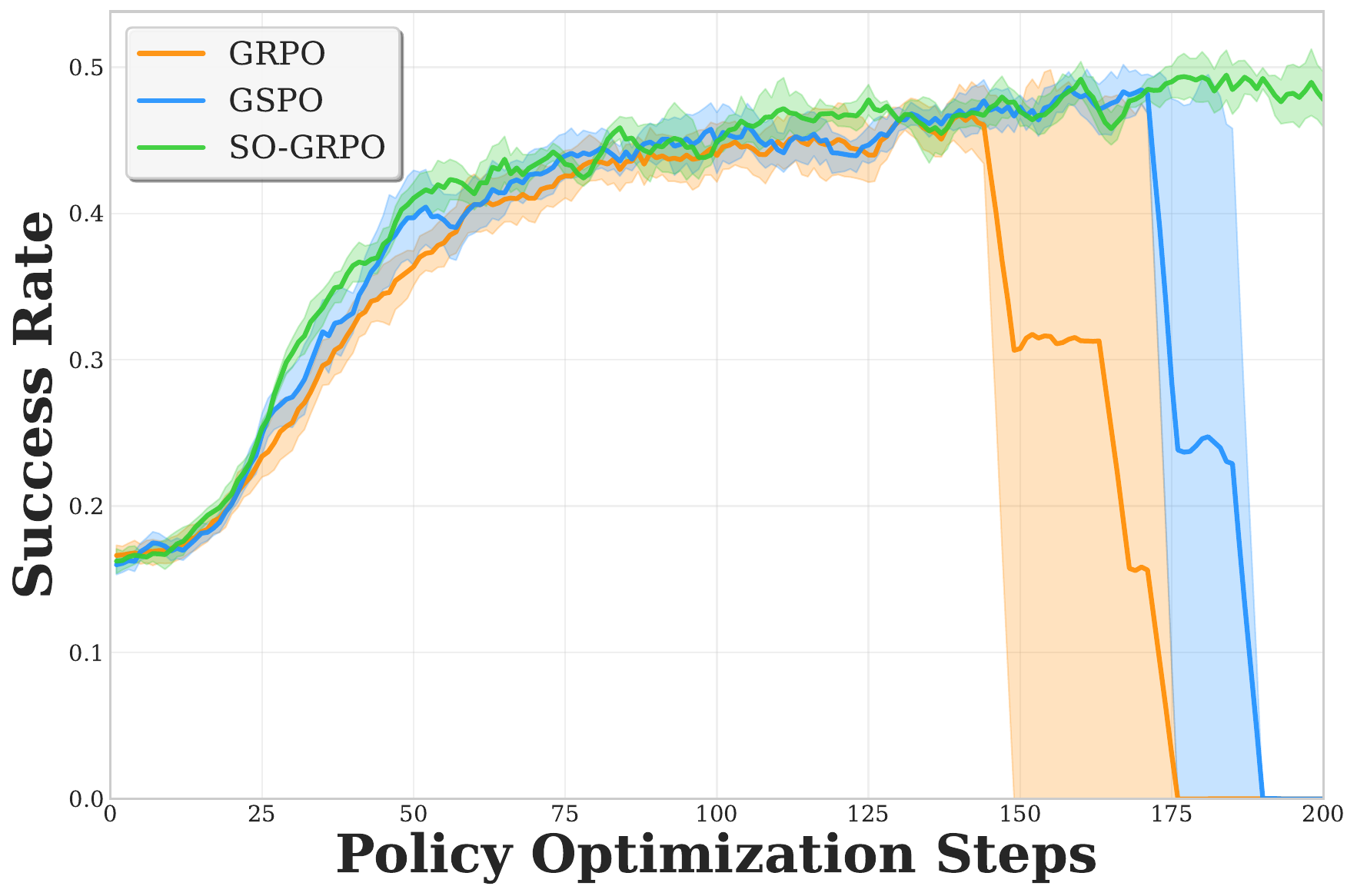}
        \caption{NQ dataset.}
        \label{fig:GRPO_performance}
    \end{subfigure}
    \begin{subfigure}{0.24\linewidth}
        \centering
        \includegraphics[width=\linewidth]{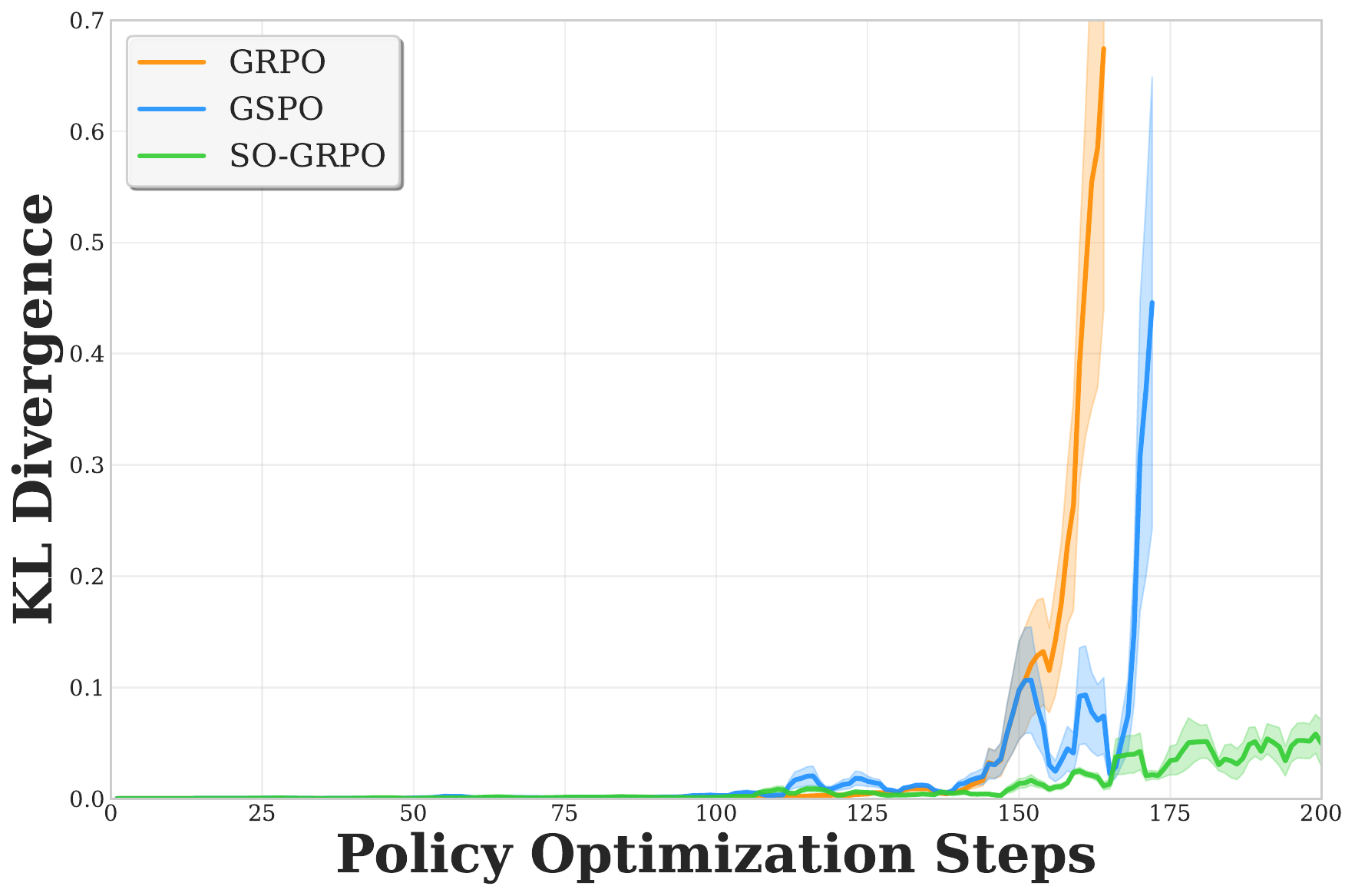}
        \caption{ {KL divergence}.}
        \label{fig:GRPO_Policy_gradient}
    \end{subfigure}
    \caption{Experimental results on  {Qwen2.5-7B} base model. (a,b) Comparison of PPO and \pponame{} on the NQ and HotpotQA datasets, where \pponame{} achieves more stable training and avoids performance collapse. (c,d) Comparison of GRPO variants on the NQ dataset. (c) Success rate shows that GRPO and GSPO suffer from instability, while \grponame{} remains stable. (d) KL divergence exhibits large spikes for GRPO and GSPO, whereas \grponame{} stays near zero. Results are averaged over three runs.}
    \label{fig:Main_result}
\end{figure*}
\section{Experiments}
We conduct comprehensive numerical evaluations of the proposed
\aname{} framework and its two instantiations, \pponame{} and
\grponame{}. We compare their performance against Search-R1
\citep{jin2025search} and other policy optimization baselines.
Our experimental results demonstrate two key advantages of the proposed
approach: (1) the combination of turn-level importance sampling and
clipping-triggered normalization significantly enhances training stability
compared to standard PPO, preventing the performance collapses commonly
observed in multi-turn agent training; and (2) the proposed stabilization
mechanism generalizes beyond PPO and remains effective across different
algorithms, training frameworks, and task domains. In addition to the original multi-turn search and medical QA experiments,
we evaluate \pponame{} within the  {AReaL asynchronous RL} training framework~\citep{fu2026areal} on mathematical reasoning. Further ablations of optimizer gradient clipping, the
normalization threshold $\delta$, and other off-policy training settings are provided in Appendix~\ref{app:sec_ablation}.

\textbf{Experimental Setup.}
Following the experimental setup of Search-R1~\citep{jin2025search},
we conduct our experiments using the {Qwen2.5-7B} base model~\citep{qwen2024qwen25} as the initial policy. For retrieval, we employ the
2018 Wikipedia dump~\citep{karpukhin2020dense} as our knowledge source
and E5 embeddings~\citep{wang2022text} as the retriever. To ensure a
fair comparison with existing methods, we follow~\citet{lin2023ra}
and consistently retrieve three passages across all retrieval-based
approaches. We evaluate \aname{} on two benchmark datasets:
(1) Natural Questions (NQ)~\citep{kwiatkowski2019natural}, and
(2) HotpotQA~\citep{yang2018hotpotqa}. {NQ evaluates single-hop open-domain retrieval and QA, whereas HotpotQA requires multi-hop retrieval and reasoning across supporting documents. Together, these benchmarks pose diverse search challenges that demand effective
coordination between information retrieval and reasoning steps}, making
them suitable testbeds for evaluating the stability and effectiveness
of our turn-level optimization approach. We evaluate on the test sets
of NQ and HotpotQA and use Exact Match (EM) as the primary metric.
Complete implementation details are provided in
Appendix~\ref{app:implement_details}.

\textbf{Training Stability on Multi-Turn Search Tasks.}
Figures~\ref{fig:PPO_nq} and~\ref{fig:PPO_hotpotqa} compare the training
dynamics of standard PPO and \pponame{} on the NQ and HotpotQA datasets.
Both methods exhibit rapid initial improvement, indicating that PPO can
identify effective policies early in training. However, their long-term
behaviors diverge substantially.

Token-level PPO shows pronounced instability: after reaching its peak
performance, the policy collapses sharply, resulting in performance
that is significantly worse than that of earlier checkpoints. This
degradation is consistent with our analysis, which shows that highly
off-policy training batches with large  {importance ratios} can
destabilize gradient updates. As a result, careful early stopping is
required to retain the best-performing checkpoint, complicating the
training procedure.

In contrast, \pponame{} maintains steady performance throughout the
optimization horizon without experiencing such collapses. We also compare turn-level PPO against a sequence-level PPO variant in Appendix~\ref{app:design_ablation}, where  {turn-level importance sampling} achieves stronger and more stable performance. These results
show that combining turn-level importance sampling with clipping-triggered
normalization provides substantially more reliable optimization than
standard token-level PPO. Additional experiments under further off-policy
training settings are reported in {Appendix~\ref{app:more_offpolicy}}.

\textbf{Extension to GRPO.}
To further validate the generality of clipping-triggered normalization,
we compare GRPO, GSPO~\citep{zheng2025group}, and the proposed
\grponame{} on the NQ dataset. GSPO is a recently proposed stabilization
method that replaces token-level importance ratios in GRPO with
sequence-level likelihood ratios. As shown in Figures~\ref{fig:GRPO_performance}
and~\ref{fig:GRPO_Policy_gradient}, \grponame{} achieves substantially
improved training stability over GRPO and further outperforms GSPO.
While GSPO improves stability by adopting sequence-level importance
sampling, it still exhibits late-stage degradation. In contrast,
\grponame{} remains stable throughout training and maintains controlled
policy updates.

These results demonstrate that clipping-triggered normalization is not
specific to PPO. It can be incorporated into other clipping-based policy
optimization algorithms as a lightweight mechanism for suppressing
unreliable off-policy updates.

\textbf{{Medical Question Answering.}}
We evaluate \pponame{} and \grponame{} on medical tasks to assess
their generalization to a domain distinct from open-domain search.
Specifically, we use the AlphaMed19K~\citep{liu2025beyond} dataset for
training, with Wikipedia serving as the retrieval source. To ensure
consistency across experiments and datasets, we adopt the prompt template
shown in {Appendix~\ref{app:medical_prompt}} for all training runs.
The hyperparameter configuration for the medical dataset follows the
same settings used in the NQ and HotpotQA experiments. Further
implementation details are provided in
Appendix~\ref{app:implement_details}.

Table~\ref{tab:medical-benchmarks-safe-min} demonstrates that our
training approach not only stabilizes multi-turn optimization but also
enables models that initially lack tool-use capability to reliably
perform retrieval and acquire domain knowledge. The poor performance
of RAG illustrates this point: introducing retrieval without
reinforcement learning is not sufficient. In contrast, both \pponame{} and \grponame{} substantially improve
performance across the medical QA benchmarks. Specifically, \pponame{}
achieves the best average accuracy among PPO-based retrieval-enhanced
methods, consistently outperforming the Search-R1 baseline
on both in-domain and out-of-domain evaluations. Moreover, \grponame{}
further improves overall performance, attaining the highest average
accuracy ({$55.33\%$}) across all retrieval-augmented and RL-enhanced
approaches.

\begin{table*}[!t]
\centering
\footnotesize
\setlength{\tabcolsep}{4.5pt} 
\renewcommand{\arraystretch}{1.08}
\caption{
Performance of 8B medical LLMs across {in-domain and out-of-domain}
benchmarks (accuracy \%).
We record model performance over the course of training. For Search-R1, the best
checkpoint corresponds to the highest-performing model observed within the first
epoch. Bold numbers indicate the best performance, while underlined numbers
indicate the second-best performance.
}
\label{tab:medical-benchmarks-safe-min}

\resizebox{0.98\textwidth}{!}{%
\begin{tabular}{lcccccc}
\toprule
\multirow{2}{*}{\textbf{Model}}
& \multicolumn{2}{c}{\textbf{{In-Domain}}}
& \multicolumn{3}{c}{\textbf{{Out-of-Domain}}}
& \multirow{2}{*}{\textbf{Avg.}} \\
\cmidrule(lr){2-3} \cmidrule(lr){4-6}
& MedQA & MedMCQA & PubMedQA & MMLU-M &  {MedXpertQA} & \\
\midrule

\rowcolor{gray!12}
\multicolumn{7}{c}{\textit{Inference-Based Methods (Base: Llama-3.1-8B-Instruct)}} \\
\midrule
Direct Inference
& 45.20 & 52.40 & 62.00 & 40.70 & 13.00 & 42.66 \\
CoT
& 48.62 & 59.80 & 64.40 & 54.20 & 13.02 & 48.01 \\
MedLlama-3-8B (CoT)
& \underline{{66.60}} & 53.40 & 64.40 & 45.70 & 11.04 & 48.23 \\

\midrule
\rowcolor{gray!12}
\multicolumn{7}{c}{
\textit{Retrieval-Augmented and RL-Enhanced Methods (Base: Llama-3.1-8B-Instruct)}
} \\
\midrule

RAG
& 8.90 & 11.30 & 16.80 & 4.90 & 9.20 & 10.22 \\

Search-R1 (0.5 epoch)
& 53.60 & 47.50 & 59.50 & 50.91 & 11.34 & 44.57 \\

Search-R1 (Best checkpoint)
& 53.40 & 54.00 & 60.20 & 49.50 & 13.20 & 46.06 \\

\pponame~(0.5 epoch)
& 48.70 & 59.30 & 60.37 & 48.04 & 8.97 & {45.08} \\

\grponame~(0.5 epoch)
& 57.73 & 60.50 & 64.04 & 56.84 & 13.51 & 50.52 \\

GSPO (1 epoch)
& 62.84 & 59.60 & 66.51 & 55.13 & 14.04 & {51.62} \\

\rowcolor{blue!5}
\pponame~(1 epoch)
& {64.80}
& \underline{60.80}
& \underline{68.31}
& \underline{59.78}
& \underline{15.75}
& \underline{{53.89}} \\

\rowcolor{blue!5}
\grponame~(1 epoch)
& \textbf{67.32}
& \textbf{62.40}
& \textbf{68.65}
& \textbf{61.24}
& \textbf{17.06}
& \textbf{{55.33}} \\

\bottomrule
\end{tabular}%
}
\end{table*}

\textbf{{Generalization to an Asynchronous RL Framework.}}
To examine whether the proposed stabilization mechanism transfers beyond multi-turn retrieval agents, we implement \pponame{} within the AReaL training framework~\citep{fu2026areal}. We train Qwen3-4B on DAPO-Math-17k~\citep{yu2025dapo} and validate on AIME-2024 over three random seeds. We compare \pponame{} against PPO with conventional optimizer gradient clipping and GSPO, while keeping the remaining training and evaluation settings fixed across methods. As shown in Figure~\ref{fig:areal_aime24}, all three methods make similar progress during the early stage of optimization, but their trajectories diverge substantially in later stages. \pponame{} continues to improve and remains stable through the end of training. In contrast, GSPO reaches a lower final training accuracy, while PPO exhibits a pronounced late-stage performance degradation despite gradient clipping.

This experiment provides two additional observations. First, the
stabilization benefit of \aname{} transfers from multi-turn retrieval
tasks to mathematical reasoning. Second, conventional optimizer gradient
clipping alone is insufficient to prevent degradation caused by unreliable
off-policy updates. We investigate this distinction through a controlled
gradient-clipping ablation in
Appendix~\ref{app:grad_clip_ablation}.

\begin{figure}[!htbp]
    \centering
    \begin{subfigure}[t]{0.48\linewidth}
        \centering
        \includegraphics[width=\linewidth]{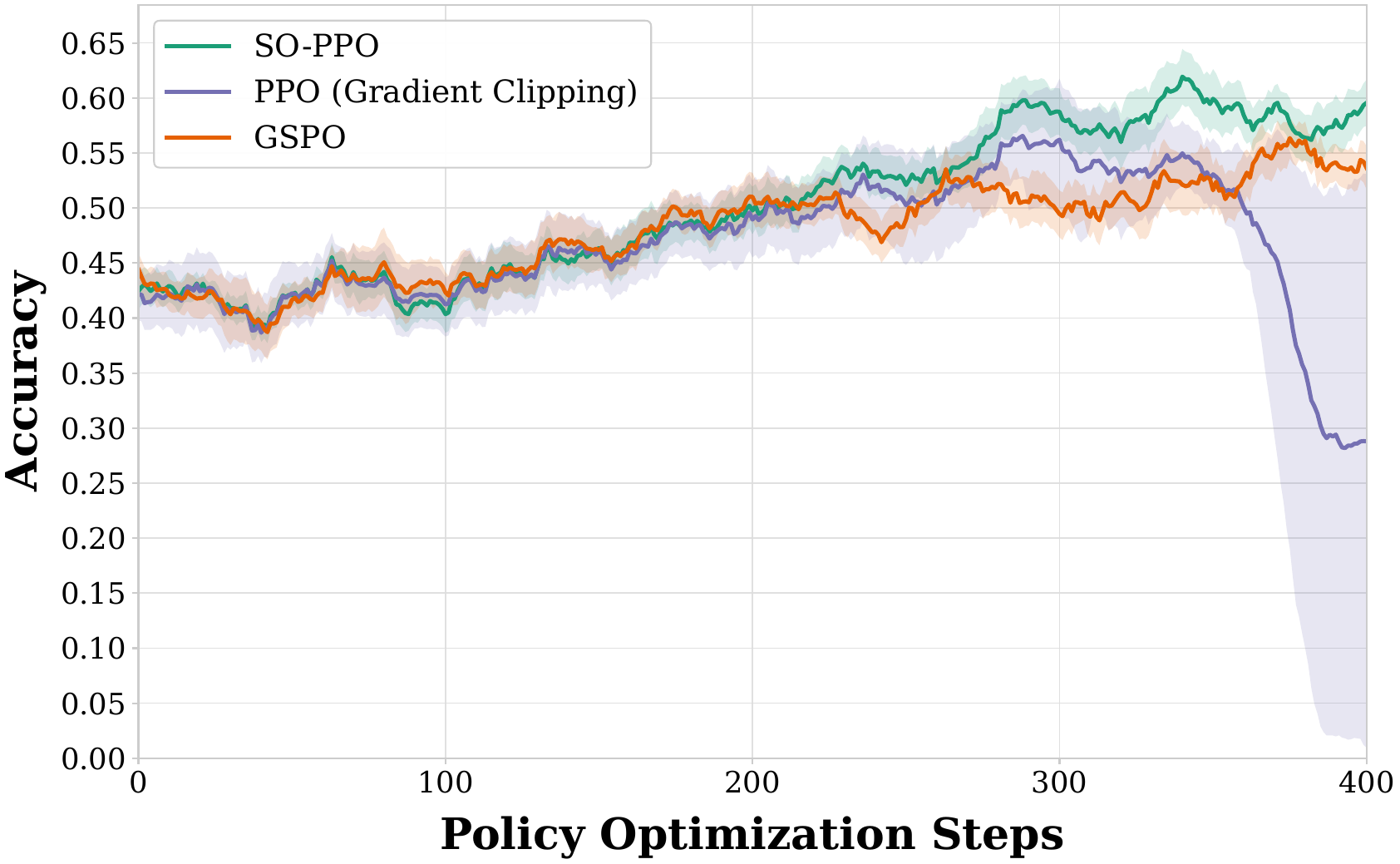}
        \caption{{Training on DAPO-Math-17k.}}
    \end{subfigure}
    \hfill
    \begin{subfigure}[t]{0.48\linewidth}
        \centering
        \includegraphics[width=\linewidth]{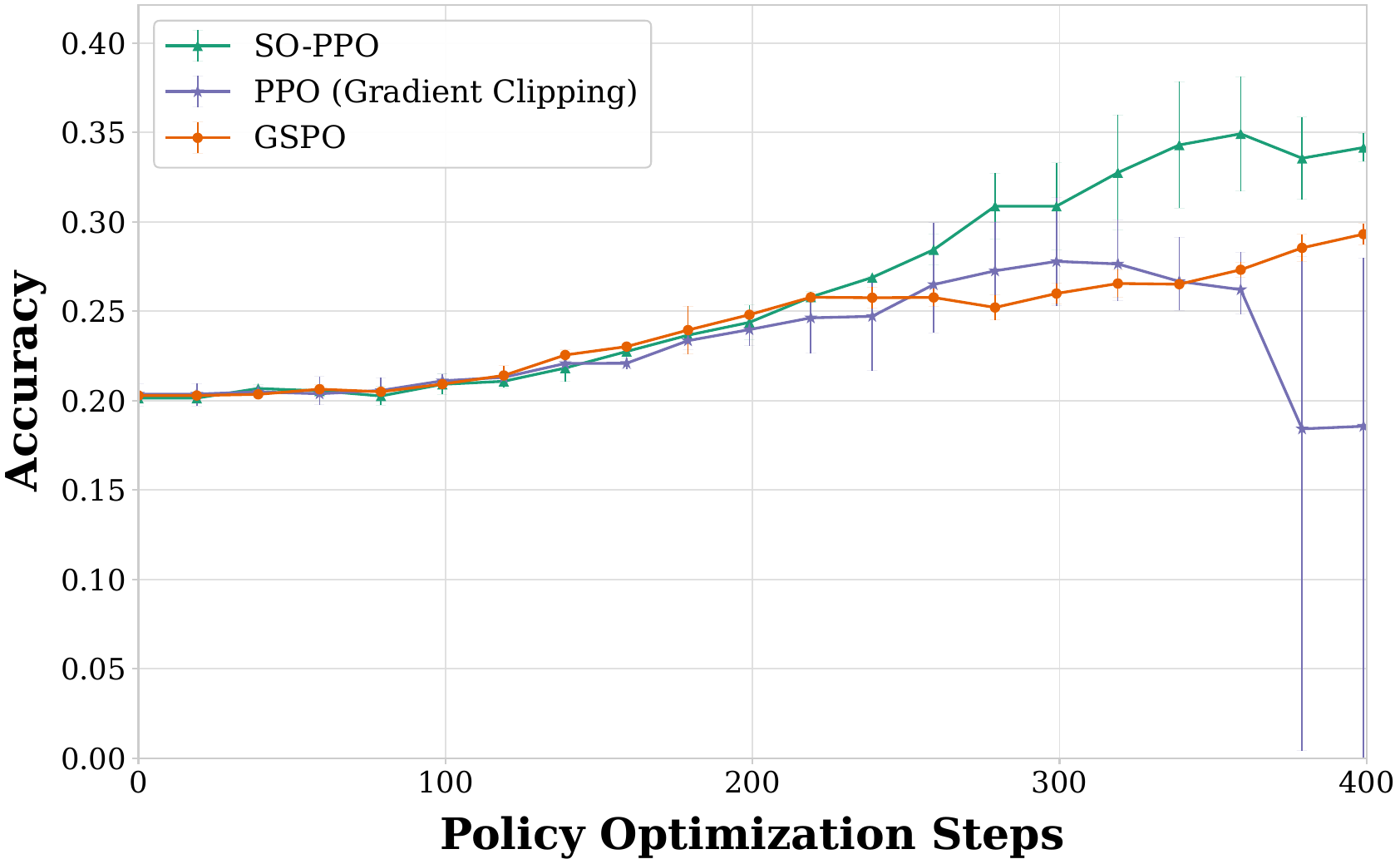}
        \caption{{Validation on AIME-2024.}}
    \end{subfigure}
    \caption{{Results under the AReaL asynchronous RL framework, averaged over three seeds. \pponame{} improves consistently and remains stable, whereas PPO with conventional optimizer gradient clipping undergoes substantial late-stage performance degradation.}}
    \label{fig:areal_aime24}
\end{figure}

\section{Conclusion}
We investigated the instability of off-policy reinforcement learning in multi-turn LLM agent training and identified two key causes: the mismatch between token-level optimization and turn-structured interactions, and the accumulation of high-variance updates under increasingly off-policy training. To address these issues, we proposed \textbf{\aname}, a framework that combines  {turn-level importance sampling} with clipping-triggered normalization to stabilize policy optimization. Instantiations of \aname{} with PPO and GRPO effectively suppress unreliable off-policy updates and yield more stable optimization dynamics. Experiments on multi-turn search and medical QA benchmarks demonstrate that \aname{} consistently improves training stability and robustness without requiring heuristic tuning or early stopping. Additional experiments in the AReaL asynchronous RL framework on DAPO-Math-17k and AIME-2024 show that the stabilization benefit also transfers to mathematical reasoning, while the gradient-clipping ablation demonstrates that it does not rely on conventional optimizer gradient clipping.

\clearpage
\IfFileExists{arxiv/references.bib}{%
  \bibliography{arxiv/references}
  \bibliographystyle{arxiv/references}
}{%
  \bibliography{references}
  \bibliographystyle{references}
}

\clearpage
\appendix
\section{Appendix}
\subsection{Representative Failure Cases}\label{app:fail_example_text}
To provide concrete evidence for the instability discussed above, we present representative failure cases from failed training runs. In the visualization, green tokens indicate samples that are clipped with  {importance ratios} smaller than $1-\epsilon$, while red tokens correspond to clipped samples with ratios larger than $1+\epsilon$. We observe that the majority of clipped tokens are associated with format-related symbols, such as angle brackets and other structural markers, rather than semantic content tokens. This observation suggests that token-level credit assignment is particularly prone to overestimating or underestimating the importance of frequently occurring format tokens, whose repetition amplifies the variance of importance sampling. As a consequence, the policy is driven to over-optimize these non-semantic tokens.  {At policy optimization step 247}, the resulting logical inconsistency in the generated output (e.g., ``worsened by eating and improved by eating'') arises from unstable gradient updates under the token-level PPO method, where high-variance off-policy importance sampling produces unreliable gradient updates.
This effect becomes evident as the generated output gradually violates the predefined formatting rules, leading to failed tool invocations. Once the format constraints are broken, downstream tool calls can no longer be executed correctly, ultimately triggering training collapse.

\begin{tcolorbox}[
  title={ {Generation at policy optimization step 243}},
  colback=blue!5,
  colframe=blue!70!black,
  boxrule=1pt,
  arc=4pt,
  left=2mm,
  right=2mm,
  top=1mm,
  bottom=1mm,
  breakable
]

\ttfamily\small

<think>

The patient is experiencing lower abdominal pain, vaginal discharge, painful intercourse, and painful urination, which are common symptoms of a genital infection.

I'm going to <search> pelvic inflammatory disease </search\redchar{>}\greenchar{Based} on the symptoms and the \redchar{result}

I'm going to <search> \greenchar{PID} diagnosis </search>\greenchar{Multiple} pathogens can cause \greenchar{pelvic}<think>

\redchar{I}'m going to <search\redchar{>} chlamydial infections </search>, there was some thought that """Chlamydia"<think>

\greenchar{Based} on \greenchar{clinical} \greenchar{findings} \greenchar{and} \greenchar{symptoms}<\redchar{answer}>E</answer\redchar{>}

\end{tcolorbox}

\begin{tcolorbox}[
  title={ {Generation at policy optimization step 247}},
  colback=blue!5,
  colframe=blue!70!black,
  boxrule=1pt,
  arc=4pt,
  left=2mm,
  right=2mm,
  top=1mm,
  bottom=1mm,
  breakable
]

\ttfamily\small

\texttt{<think>}

The patient is experiencing symptoms of gastroparesis, including bloating, nausea, and dark stools, which are worsened by eating and improved by eating.

I'll \texttt{<search>} gastroparesis \texttt{</search>}
\greenchar{assistant} \texttt{<think>}

\redchar{\$}\texttt{search>} gastroesophageal reflux disease
\greenchar{complications}\texttt{</search>}\texttt{<think>}

\redchar{\texttt{<b}}<think{\greenchar{>}}

assistant \greenchar{')},

\greenchar{\texttt{<Bassistant}}\greenchar{\texttt{<think>}}

\greenchar{assistant}

\end{tcolorbox}

\subsection{Proof of Lemma~\ref{lm:proposed_ppo_grad}}
\label{app:proposed_ppo_grad}
The proof establishes turn-level credit assignment by first reformulating the PPO objective with indicator functions, then computing how the gradient with respect to a single token depends only on its containing turn due to locality. The key insight is that the derivative of the turn-level importance weight yields a scaling factor of $\frac{w_k^{\text{turn}}(\theta)}{|y^k|}$, where the $\frac{1}{|y^k|}$ term arises from the geometric mean structure of the turn-level importance ratio. Finally, aggregating over all tokens via the multivariable chain rule produces the claimed gradient form, where each turn receives credit proportional to $w_k^{\text{turn}}(\theta)\frac{\hat{A}^k}{|y^k|}$, combining importance weighting, accumulated advantage, and length normalization.

We now proceed with the lemma proof. We derive the gradient of the Turn-PPO objective $\mathcal{J}_{\mathrm{Turn-PPO}}(\theta)$ defined in Eq.~\ref{eq:turn_level_PPO}. We begin with the objective in expanded form:
\begin{align}
    &\mathcal{J}_{\mathrm{Turn-PPO}}(\theta)
    \nonumber\\
    &= \mathbb{E}_{x \sim \mathcal{D}, \ y \sim \pi_{\theta_{\mathrm{old}}}(\cdot|x)} \left[\frac{1}{|y|} \sum_{k=1}^{K} \sum_{t=t_k^{\text{start}}}^{t_k^{\text{end}}} \min\left\{ w_k^{\text{turn}}(\theta) \hat{A}_{t}, \text{clip}(w_k^{\text{turn}}(\theta), 1-\epsilon, 1+\epsilon) \hat{A}_{t} \right\} \right]
    \nonumber\\
    &=\mathbb{E}_{x \sim \mathcal{D}, \ y \sim \pi_{\theta_{\mathrm{old}}}(\cdot|x)} \left[\frac{1}{|y|} \sum_{k=1}^{K} \sum_{t=t_k^{\text{start}}}^{t_k^{\text{end}}}
    \left(
    \mathds{1}_{(k,t) \in \mathcal{B}_{\text{turn}}}  w_k^{\text{turn}}(\theta) \hat{A}_{t}
    \right.
    \right.
    \nonumber\\
    &
    \Biggl.
    \left.
    \hspace{53mm} \ \
    + \mathds{1}_{(k,t)\in \mathcal{B}_{\text{turn}}^c}  \text{clip}(w^{\text{turn}}_k(\theta), 1-\epsilon, 1+\epsilon) \hat{A}_t
    \right)
    \Biggr],
    \label{eqn:ppo_objective_no_clip}
\end{align}
where\footnote{For notational brevity, we drop the expectation subscript in what follows, with all expectations taken over the same joint distribution where $x \sim \mathcal{D}$ and $y \sim \pi_{\theta_{\mathrm{old}}}(\cdot|x)$.} Eq.~\ref{eqn:ppo_objective_no_clip} follows from decomposing the min function using indicator functions.
The indicator function $\mathds{1}_{(k,t) \in \mathcal{B}_{\text{turn}}}$ equals $1$ if the token index $t$ within turn $k$ belongs to the set $\mathcal{B}_{\text{turn}}$, and $0$ otherwise.
The set $\mathcal{B}_{\text{turn}} = \mathcal{B}_{\text{turn}}^+ \cup \mathcal{B}_{\text{turn}}^-$ represents the indices of tokens within turns where clipping is inactive.
Together, the set $\mathcal{B}_{\text{turn}}$ and its complement  {$\mathcal{B}_{\text{turn}}^c$} partition the space of token indices across all turns, satisfying $\mathcal{B}_{\text{turn}} \cup \mathcal{B}_{\text{turn}}^c = \{1, 2, \ldots, |y|\}$ and $\mathcal{B}_{\text{turn}} \cap \mathcal{B}_{\text{turn}}^c = \emptyset$.
For token indices $(k,t) \in \mathcal{B}_{\text{turn}}$, the min operation selects the unclipped term $w_k^{\text{turn}}(\theta) \hat{A}_{t}$; otherwise, for $(k,t) \in \mathcal{B}_{\text{turn}}^c$, it selects the clipped term $\text{clip}(w_k^{\text{turn}}(\theta), 1-\epsilon, 1+\epsilon) \hat{A}_{t}$.

Taking the gradient with respect to $\log \pi_\theta(y_{t'} | x, y_{<t'})$ where
$t' \in [t_{k'}^{\text{start}}, t_{k'}^{\text{end}}]$ (i.e., token $t'$ belongs to turn $k'$), we get
\begin{align}
    \frac{\partial \mathcal{J}_{\mathrm{Turn-PPO}}(\theta)}{\partial \log \pi_\theta(y_{t'} | x, y_{<t'})}
    &=
    \mathbb{E} \left[\frac{1}{|y|} \sum_{k=1}^{K} \sum_{t=t_k^{\text{start}}}^{t_k^{\text{end}}} \mathds{1}_{(k,t) \in \mathcal{B}_{\text{turn}}} \ \frac{\partial w_k^{\text{turn}}(\theta)}{\partial \log \pi_\theta(y_{t'} | x, y_{<t'})} \hat{A}_{t} \right]
    \label{eqn:grad_turn_ppo_1}\\
    &=
    \mathbb{E} \left[ \frac{1}{|y|} \sum_{t=t_{k'}^{\text{start}}}^{t_{k'}^{\text{end}}} \mathds{1}_{ {(k',t)} \in \mathcal{B}_{\text{turn}}} \ \frac{\partial w_{k'}^{\text{turn}}(\theta)}{\partial \log \pi_\theta(y_{t'} | x, y_{<t'})} \hat{A}_{t} \right],
    \label{eqn:grad_turn_ppo_2}
\end{align}
where Eq.~\ref{eqn:grad_turn_ppo_1} follows since clipped terms have zero gradients by definition of the clipping function; and Eq.~\ref{eqn:grad_turn_ppo_2} follows since $w_k^{\text{turn}}(\theta)$ depends only on tokens within turn $k$. Since token $t'$ belongs to turn $k'$, we have $\frac{\partial w_k^{\text{turn}}(\theta)}{\partial \log \pi_\theta(y_{t'} | x, y_{<t'})} = 0$ for all $k \neq k'$  due to disjoint turn boundaries.

Next, we evaluate the derivative of $w_{k'}^{\text{turn}}(\theta)$ with respect to the log probability. Recall that $w_{k'}^{\text{turn}}(\theta)$ is defined as
\begin{align}
    w_{k'}^{\text{turn}}(\theta)
    &=
    \left( \frac{\pi_\theta(y^{k'} | x, y^{<k'})}{\pi_{\theta_{\mathrm{old}}}(y^{k'} | x, y^{<k'})} \right)^{\frac{1}{|y^{k'}|}}
    \nonumber\\
    &=
    \left( \prod_{t=t_{k'}^{\text{start}}}^{t_{k'}^{\text{end}}} \frac{\pi_\theta(y_t | x, y_{<t})}{\pi_{\theta_{\mathrm{old}}}(y_t | x, y_{<t})} \right)^{\frac{1}{|y^{k'}|}}
    \nonumber\\
    &=
    \exp\left( \frac{1}{|y^{k'}|} \sum_{t=t_{k'}^{\text{start}}}^{t_{k'}^{\text{end}}} \left[ \log \pi_\theta(y_t | x, y_{<t}) - \log \pi_{\theta_{\mathrm{old}}}(y_t | x, y_{<t}) \right] \right)
    \nonumber\\
    &=
    \exp\left( -\frac{1}{|y^{k'}|} \sum_{t=t_{k'}^{\text{start}}}^{t_{k'}^{\text{end}}} \log \pi_{\theta_{\mathrm{old}}}(y_t | x, y_{<t}) \right) \cdot \exp\left( \frac{1}{|y^{k'}|} \sum_{t=t_{k'}^{\text{start}}}^{t_{k'}^{\text{end}}} \log \pi_\theta(y_t | x, y_{<t}) \right),
    \label{eqn:weight_turn_ppo}
\end{align}
where the first exponential term in Eq.~\ref{eqn:weight_turn_ppo} is constant with respect to $\theta$.
We can write $w_{k'}^{\text{turn}}(\theta) = \kappa \cdot \exp{\left(f(\theta)\right)}$ where
\begin{align}
    \kappa &= \exp\left( -\frac{1}{|y^{k'}|} \sum_{t=t_{k'}^{\text{start}}}^{t_{k'}^{\text{end}}} \log \pi_{\theta_{\mathrm{old}}}(y_t | x, y_{<t}) \right),
    \nonumber\\
    f(\theta) &= \frac{1}{|y^{k'}|} \sum_{t=t_{k'}^{\text{start}}}^{t_{k'}^{\text{end}}} \log \pi_\theta(y_t | x, y_{<t}).
\end{align}
Consequently, the derivative of $w_{k'}^{\text{turn}}(\theta)$ with respect to the log probability is evaluated as
\begin{align}
    \frac{\partial w_{k'}^{\text{turn}}(\theta)}{\partial \log \pi_\theta(y_{t'} | x, y_{<t'})}
    &=
    \kappa \cdot \exp{\left(f(\theta)\right)} \cdot \frac{\partial f(\theta)}{\partial \log \pi_\theta(y_{t'} | x, y_{<t'})}
    \label{grad_weight_turn_ppo_1}\\
    &=
    w_{k'}^{\text{turn}}(\theta) \cdot \frac{\partial}{\partial \log \pi_\theta(y_{t'} | x, y_{<t'})} \left[ \frac{1}{|y^{k'}|} \sum_{t=t_{k'}^{\text{start}}}^{t_{k'}^{\text{end}}} \log \pi_\theta(y_t | x, y_{<t}) \right]
    \nonumber\\
    &=
    w_{k'}^{\text{turn}}(\theta) \cdot \frac{1}{|y^{k'}|} \cdot \frac{\partial \log \pi_\theta(y_{t'} | x, y_{<t'})}{\partial \log \pi_\theta(y_{t'} | x, y_{<t'})}
    \label{grad_weight_turn_ppo_2}\\
    &=
    w_{k'}^{\text{turn}}(\theta) \cdot \frac{1}{|y^{k'}|},
    \label{grad_weight_turn_ppo_3}
\end{align}
where Eq.~\ref{grad_weight_turn_ppo_1} follows from the chain rule; and Eq.~\ref{grad_weight_turn_ppo_2} follows since only the term with index $t = t'$ in the summation depends on $\log \pi_\theta(y_{t'} | x, y_{<t'})$.
Substituting Eq.~\ref{grad_weight_turn_ppo_3} back into Eq.~\ref{eqn:grad_turn_ppo_2}, we get
\begin{align}
    \frac{\partial \mathcal{J}_{\mathrm{Turn-PPO}}(\theta)}{\partial \log \pi_\theta(y_{t'} | x, y_{<t'})}
    =
    \mathbb{E} \left[ \frac{w_{k'}^{\text{turn}}(\theta)}{|y| \cdot |y^{k'}|} \sum_{t=t_{k'}^{\text{start}}}^{t_{k'}^{\text{end}}}\mathds{1}_{(k',t) \in \mathcal{B}_{\text{turn}}} \hat{A}_{t} \right].
    \label{eqn:partial_derivative_exp}
\end{align}

Finally, from Eq.~\ref{eqn:ppo_objective_no_clip}, the full gradient is given by
\begin{align}
    \nabla_{\theta} \mathcal{J}_{\mathrm{Turn-PPO}}(\theta)
    &= \sum_{k=1}^{K} \sum_{t=t_k^{\text{start}}}^{t_k^{\text{end}}} \frac{\partial \mathcal{J}_{\mathrm{Turn-PPO}}(\theta)}{\partial \log \pi_\theta(y_t | x, y_{<t})} \nabla_{\theta} \log \pi_\theta(y_t | x, y_{<t})
    \label{eqn:full_gradient_1}\\
    &= \sum_{k=1}^{K} \sum_{t=t_k^{\text{start}}}^{t_k^{\text{end}}} \mathbb{E} \left[ \frac{w_k^{\text{turn}}(\theta)}{|y| \cdot |y^k|} \sum_{j=t_k^{\text{start}}}^{t_k^{\text{end}}}\mathds{1}_{(k,j) \in \mathcal{B}_{\text{turn}}} \hat{A}_{j} \right] \nabla_{\theta} \log \pi_\theta(y_t | x, y_{<t})
    \label{eqn:full_gradient_2}\\
    &= \mathbb{E} \left[\sum_{k=1}^{K} \frac{w_k^{\text{turn}}(\theta)}{|y| \cdot |y^k|} \left(\sum_{j=t_k^{\text{start}}}^{t_k^{\text{end}}}\mathds{1}_{(k,j) \in \mathcal{B}_{\text{turn}}} \hat{A}_{j}\right) \sum_{t=t_k^{\text{start}}}^{t_k^{\text{end}}} \nabla_{\theta} \log \pi_\theta(y_t | x, y_{<t}) \right]
    \nonumber\\
    &= \mathbb{E} \left[\sum_{k=1}^{K} \frac{w_k^{\text{turn}}(\theta)}{|y| \cdot |y^k|} \left(\sum_{j=t_k^{\text{start}}}^{t_k^{\text{end}}}\mathds{1}_{(k,j) \in \mathcal{B}_{\text{turn}}} \hat{A}_{j}\right) \nabla_{\theta} \log \pi_\theta(y^k | x, y^{<k}) \right]
    \label{eqn:full_gradient_4}\\
    &= \mathbb{E} \left[\frac{1}{|y|} \sum_{k=1}^{K}
    w_k^{\text{turn}}(\theta)\frac{\hat{A}^k}{|y^k|}
    \nabla_{\theta} \log \pi_\theta(y^k | x, y^{<k}) \right],
    \label{eqn:full_gradient_5}
\end{align}
where Eq.~\ref{eqn:full_gradient_1} applies the multivariable chain rule; Eq.~\ref{eqn:full_gradient_2} follows from Eq.~\ref{eqn:partial_derivative_exp}; Eq.~\ref{eqn:full_gradient_4} uses the chain rule identity $\nabla_{\theta} \log \pi_\theta(y^k | x, y^{<k}) = \sum_{t=t_k^{\text{start}}}^{t_k^{\text{end}}} \nabla_{\theta} \log \pi_\theta(y_t | x, y_{<t})$; and Eq.~\ref{eqn:full_gradient_5} follows by defining $\hat{A}^k \coloneqq \sum_{j=t_k^{\text{start}}}^{t_k^{\text{end}}} \mathds{1}_{(k,j) \in \mathcal{B}_{\text{turn}}} \hat{A}_{j}$.
This completes the proof of Lemma~\ref{lm:proposed_ppo_grad}.
\hfill $\blacksquare$

\subsection{Proof of Lemma~\ref{lm:ppo_grad_decomp}}
\label{app:ppo_grad_decomp}
The token-level PPO objective function can be expressed as
\begin{align}
\mathcal{J}_{\mathrm{PPO}}(\theta)
&=
\mathbb{E}_{x \sim  {\mathcal{D}}, \ y \sim \pi_{\theta_{\mathrm{old}}}(\cdot|x)} \!\left[ \frac{1}{|y|} \sum_{t=1}^{|y|} \min(w_t(\theta) \hat{A}_t, \text{clip}(w_t(\theta), 1-\epsilon, 1+\epsilon) \hat{A}_t) \right]
\nonumber\\
&=
\mathbb{E}_{x \sim  {\mathcal{D}}, \ y \sim \pi_{\theta_{\mathrm{old}}}(\cdot|x)} \!\left[ \frac{1}{|y|} \sum_{t=1}^{|y|} \left( \mathds{1}_{t \in \mathcal{B}} \cdot w_t(\theta) \hat{A}_t + \mathds{1}_{t \in \mathcal{B}^c} \cdot \text{clip}(w_t(\theta), 1-\epsilon, 1+\epsilon) \hat{A}_t \right) \right]\!,
\label{eqn:ppo_fn_indicator}
\end{align}
where Eq.~\ref{eqn:ppo_fn_indicator} follows from the definition of the event $\mathcal{B}$ and
\begin{align}
    \text{clip}(w, 1-\epsilon, 1+\epsilon)
    = \max(\min(w, 1+\epsilon), 1-\epsilon)
    =
    \begin{cases}
        1-\epsilon & \text{if } w < 1-\epsilon, \\
        w & \text{if } 1-\epsilon \leq w \leq 1+\epsilon, \\
        1+\epsilon & \text{if } w > 1+\epsilon.
    \end{cases}
\end{align}
Taking the gradient with respect to $\theta$, we obtain
\begin{align}
\nabla_\theta \mathcal{J}_{\mathrm{PPO}}(\theta)
&=
\mathbb{E} \left[ \frac{1}{|y|} \sum_{t=1}^{|y|} \left( \mathds{1}_{t \in \mathcal{B}} \cdot \nabla_\theta w_t(\theta) \hat{A}_t + \mathds{1}_{t \in \mathcal{B}^c} \cdot \nabla_\theta \:\text{clip}(w_t(\theta), 1-\epsilon, 1+\epsilon) \hat{A}_t \right) \right]
\nonumber\\
&=
\mathbb{E} \left[ \frac{1}{|y|} \sum_{t=1}^{|y|} \mathds{1}_{t \in \mathcal{B}} \cdot \nabla_\theta w_t(\theta) \hat{A}_t \right]
\label{eqn:grad_ppo_fn_indicator_1}\\
&=
\mathbb{E} \left[ \frac{1}{|y|} \sum_{t=1}^{|y|} \mathds{1}_{t \in \mathcal{B}} w_t(\theta) \nabla_\theta \log \pi_\theta(y_t | x, y_{<t}) \hat{A}_t \right],
\label{eqn:grad_ppo_fn_indicator_2}
\end{align}
where Eq.~\ref{eqn:grad_ppo_fn_indicator_1} follows from the definition of the event $\mathcal{B}^c$ that yields
\begin{align}
    \mathds{1}_{t \in \mathcal{B}^c} \cdot \nabla_\theta \text{clip}(w, 1-\epsilon, 1+\epsilon) \hat{A}_t
    =
    \begin{cases}
        \nabla_\theta (1-\epsilon)\hat{A}_t = 0 & \text{if } \hat{A}_t < 0 \text{ and } w_t < 1 - \epsilon, \\
        \nabla_\theta (1+\epsilon)\hat{A}_t = 0 & \text{if } \hat{A}_t \geq 0 \text{ and } w_t > 1 + \epsilon;
    \end{cases}
\end{align}
and Eq.~\ref{eqn:grad_ppo_fn_indicator_2} readily follows from the definition of $w_t(\theta)$ that leads to
\begin{align}
    \nabla_\theta w_t(\theta)
    = \nabla_\theta \frac{\pi_\theta(y_t | x, y_{<t})}{ {\pi_{\theta_{\mathrm{old}}}(y_t | x, y_{<t})}}
    = w_t(\theta) \nabla_\theta \log \pi_\theta(y_t | x, y_{<t}).
\end{align}

Finally, we verify that the RHS of Lemma~\ref{lm:ppo_grad_decomp} is equal to Eq.~\ref{eqn:grad_ppo_fn_indicator_2} as follows:
\begin{align}
    &\mathbb{E}\left[\frac{1}{|y|} \!\sum_{t=1}^{|y|} w_t \nabla_\theta \log \pi_\theta \hat{A}_t\right]
    -
    \mathbb{E}\left[\frac{1}{|y|} \sum_{t=1}^{|y|}\mathds{1}_{t \in \mathcal{B}^c} w_t \nabla_\theta \log \pi_\theta \hat{A}_t\right]
    \nonumber\\
    &=
    \mathbb{E}\left[\frac{1}{|y|} \sum_{t=1}^{|y|}\mathds{1}_{t \in \mathcal{B}} w_t \nabla_\theta \log \pi_\theta \hat{A}_t\right],
    \label{eqn:rhs_decmop_verify}
\end{align}
where Eq.~\ref{eqn:rhs_decmop_verify} follows from $(1 - \mathds{1}_{t \in \mathcal{B}^c}) = \mathds{1}_{t \in \mathcal{B}}$. This completes the proof of Lemma~\ref{lm:ppo_grad_decomp}.
\hfill $\blacksquare$

\subsection{\pponame{} Pseudocode}
\label{app:ST_PPO_algorithm}

{Algorithm~\ref{alg:ST_PPO} presents the two stages of \pponame{}: computing the unnormalized turn-level PPO gradient in Eq.~\ref{eq:turn_level_credit_assignment} and applying the mini-batch normalization in Eq.~\ref{eq:ST_PPO_mc}.}

\begin{algorithm}[H]
\caption{\textit{\pponame}}
\label{alg:ST_PPO}
\footnotesize
\setlength{\linewidth}{\hsize}
\begin{algorithmic}
    \STATE \textbf{Input:} Policy $\pi_0$, step size $\eta$, dataset $\mathcal{D}$, and maximum generation length $L$.
    \FOR{iteration $k=0,1,\ldots,K-1$}
        \STATE Sample queries $x\sim\mathcal{D}$ and trajectories $\tau\coloneqq y_{0:L}\sim\pi_k$; split each trajectory into turn-level state--action pairs.
        \STATE Compute the unnormalized turn-level PPO gradient using Eq.~\ref{eq:turn_level_credit_assignment}.
        \STATE Compute the clipping-bias norm and normalize the mini-batch update using Eq.~\ref{eq:ST_PPO_mc}.
        \STATE Update the policy with the normalized gradient and the critic with the TD error.
    \ENDFOR
\end{algorithmic}
\end{algorithm}

\subsection{\grponame{} Derivation and Pseudocode}
\label{app:SO_GRPO}

Notably, the clipping-triggered normalization can be incorporated into the GRPO framework,
yielding a stabilized variant that we refer to as \grponame.
Unlike standard PPO, GRPO uses a group-relative advantage in its surrogate objective.
We therefore forgo turn-level importance sampling in this context.
Formally, GRPO optimizes an objective of the form:
\begin{equation}\label{eq:GRPO_objective}
\mathcal{J}_{\mathrm{GRPO}}(\theta)
=
\mathbb{E}\Big[
\min\!\Big(
w(\theta)\,\hat{A}^{\mathrm{grp}},
\operatorname{clip}\!\big(w(\theta), 1-\epsilon, 1+\epsilon\big)\,\hat{A}^{\mathrm{grp}}
\Big)
\Big].
\end{equation}

Following GRPO \citep{shao2024deepseekmath}, importance sampling is performed at the token level. For a trajectory $\tau = (y_1,\ldots,y_{|y|})$ sampled from the old policy $\pi_{\theta_{\mathrm{old}}}$, the token-level importance ratio is defined as
$
w_t(\theta)
:=
\frac{\pi_\theta(y_t \mid x, y_{<t})}
     {\pi_{\theta_{\mathrm{old}}}(y_t \mid x, y_{<t})}.
$
In the GRPO surrogate objective, this token-level ratio is combined with a  {group-relative advantage} that
is shared across all tokens within the same trajectory.
Specifically, given a group of $G$ sampled responses $\{\tau_i\}_{i=1}^G$ with trajectory-level rewards
$\{r_i\}_{i=1}^G$, the  {group-relative advantage} is defined as
$
\hat{A}^{\mathrm{grp}}(\tau_i)
:=
\frac{r_i - \operatorname{mean}(\{r_j\}_{j=1}^G)}
     {\operatorname{std}(\{r_j\}_{j=1}^G)} ,
$ and is assigned uniformly to all tokens $t$ in $\tau_i$.
To apply the proposed clipping-triggered normalization mechanism, we compute the  {clipping-bias vector} using these token-level importance
ratios together with the shared  {group-relative advantages}. {Using the clipping set $\mathcal{B}$ defined analogously to PPO, the corresponding GRPO policy gradient is}
{
\begin{align}
\nabla_\theta \mathcal{J}_{\mathrm{GRPO}}(\theta)
&= \mathbb{E}_{\tau\sim\pi_{\theta_{\mathrm{old}}}}\!\left[
\frac{1}{|y|}\sum_{t=1}^{|y|}
\mathds{1}_{t\in\mathcal{B}}\,
w_t(\theta)\,\hat{A}^{\mathrm{grp}}\,
\nabla_\theta\log\pi_\theta(y_t\mid x,y_{<t})
\right].
\label{eq:GRPO_gradient}
\end{align}
}
\begin{equation*}
    C_{\mathrm{GRPO}}(\theta)
    :=
    \mathbb{E}_{\tau\sim\pi_{\theta_{\mathrm{old}}}}\!\left[
    \frac{1}{|y|}
    \sum_{t=1}^{|y|}
    \mathds{1}_{t \in \mathcal{B}^c}\,
    w_t(\theta)\,\hat{A}^{\mathrm{grp}}\,
    \nabla_\theta \log \pi_\theta(y_t\mid x,y_{<t})
    \right].
\end{equation*}

{Here, $\mathcal{B}$ is defined analogously to the clipping set in PPO, with $\hat{A}^{\mathrm{grp}}$ replacing the advantage. Its complement $\mathcal{B}^{c}$ contains the samples whose importance ratios fall outside the trust region, and $C_{\mathrm{GRPO}}(\theta)$ aggregates their gradient contribution.} The stabilized GRPO update is then defined analogously to \pponame:
\[
\nabla_\theta \mathcal{J}_{\mathrm{SO\text{-}GRPO}}(\theta)
:=
\frac{1}{\big\| C_{\mathrm{GRPO}}(\theta)\big\|_2}\,
\nabla_\theta \mathcal{J}_{\mathrm{GRPO}}(\theta).
\]
As in \pponame{}, both the GRPO gradient and its  {clipping-bias vector} admit Monte Carlo estimators
based on the same set of sampled trajectories.
Specifically, the GRPO gradient and  {clipping-bias vector} can be written in expectation as
\[
\nabla_\theta \mathcal{J}_{\mathrm{GRPO}}(\theta)
:= \mathbb{E}_{\tau \sim  {\pi_{\theta_{\mathrm{old}}}}}\!\left[ \nabla_\theta \mathcal{J}_{\mathrm{GRPO}}(\theta;\tau) \right],
\qquad
C_{\mathrm{GRPO}}(\theta)
:= \mathbb{E}_{\tau \sim  {\pi_{\theta_{\mathrm{old}}}}}\!\left[ C_{\mathrm{GRPO}}(\theta;\tau) \right].
\]
In practice, both quantities are estimated using the same mini-batch of sampled trajectories, yielding the Monte Carlo estimator
\begin{align}\label{eq:GRPO_MC}
\nabla_\theta \hat{\mathcal{J}}_{\mathrm{SO\text{-}GRPO}}(\theta)
:= \frac{1}{\max(\lVert \hat{C}_{\mathrm{GRPO}}(\theta) \rVert_2, \delta)}
\, \nabla_\theta \hat{\mathcal{J}}_{\mathrm{GRPO}}(\theta),
\end{align}
where $\hat{C}_{\mathrm{GRPO}}(\theta)$ and $\nabla_\theta \hat{\mathcal{J}}_{\mathrm{GRPO}}(\theta)$ are computed
from the same sampled groups. This ensures a consistent implementation of clipping-triggered normalization across algorithms.
\begin{algorithm}
\caption{\textit{\grponame}}
\small
\setlength{\linewidth}{\hsize}
\begin{algorithmic}
    \STATE \textbf{Input:} Initialize policy $\pi_{0}$, step size $\eta$, dataset $\mathcal{D}$, group size $G$, and maximum generation length $L$.
    \FOR{iteration $k = 0,1,\ldots,K-1$}
        \STATE \textbf{Data Sampling I:} Sample a batch of queries $x \sim \mathcal{D}$.
        \STATE \textbf{Data Sampling II:} For each query $x$, sample a group of $G$ responses $\{\tau_i \coloneqq y^{(i)}_{0:L}\}_{i=1}^G \sim \pi_k(\cdot|x)$.
        \STATE \textbf{Reward \& Advantage:} Compute rewards $\{r(x,\tau_i)\}_{i=1}^G$ and form  {group-relative advantages} $\{\hat{A}^{\mathrm{grp}}(\tau_i)\}_{i=1}^G$ (e.g., normalized within the group).
        \STATE \textbf{Gradient Estimation:} {Compute the GRPO gradient as in Eq.~\ref{eq:GRPO_gradient} (token-level ratios with shared group advantage).}
        \STATE \textbf{Clipping-Triggered Normalization:} Compute the  {clipping-bias norm} and form the surrogate gradient for \grponame\ (Eq.~\ref{eq:GRPO_MC}) using the sampled groups $\{\tau_i\}_{i=1}^G$.
        \STATE \textbf{Policy Improvement:} Update the policy using the surrogate gradient.
    \ENDFOR
\end{algorithmic}
\label{alg:SO_GRPO}
\end{algorithm}

\subsection{Supplementary Experiments}
\label{app:sec_ablation}

We provide additional ablations to examine the two components of
\aname{} and their robustness under different training configurations.
Specifically, we study:
(1) the relationship between clipping-triggered normalization and
conventional optimizer gradient clipping;
(2) the effect of importance-weight granularity and the sensitivity to
the normalization threshold $\delta$; and
(3) performance under increased off-policy mini-batch reuse.

\subsubsection{{Importance-Weight Granularity} and
Normalization Threshold}
\label{app:design_ablation}

\textbf{ {Importance-weight granularity.}}
To examine the effect of importance-weight granularity, we compare
turn-level PPO with a sequence-level PPO variant. The two methods use
the same training configuration and differ only in whether importance
weights are aggregated within each interaction turn or across the entire
generated sequence.

As shown in Figure~\ref{fig:turn_vs_seq}, both methods improve rapidly
during early training. However, turn-level PPO learns faster, achieves a
higher final success rate, and exhibits more stable late-stage
optimization. Sequence-level PPO instead converges to a lower performance
level and shows a noticeable degradation toward the end of training.
These results support the use of interaction turns as the optimization
granularity for multi-turn agent training.

\textbf{Normalization threshold.}
The constant $\delta$ lower-bounds the normalization denominator and
prevents excessively large updates when the estimated clipping-bias norm
is small. We evaluate $\delta\in\{1,5,10\}$ while keeping all other
training settings fixed.

Figure~\ref{fig:delta_ablation} shows that all three configurations
follow similar optimization trajectories and reach comparable final
success rates. None of the evaluated settings exhibits training
instability or performance collapse, indicating that \pponame{} is
robust to the choice of $\delta$ within the evaluated range. We therefore
use $\delta=1$ as the default setting in the main experiments.

\begin{figure}[H]
    \centering

    \begin{subfigure}[t]{0.49\linewidth}
        \centering
        \includegraphics[
            width=\linewidth
        ]{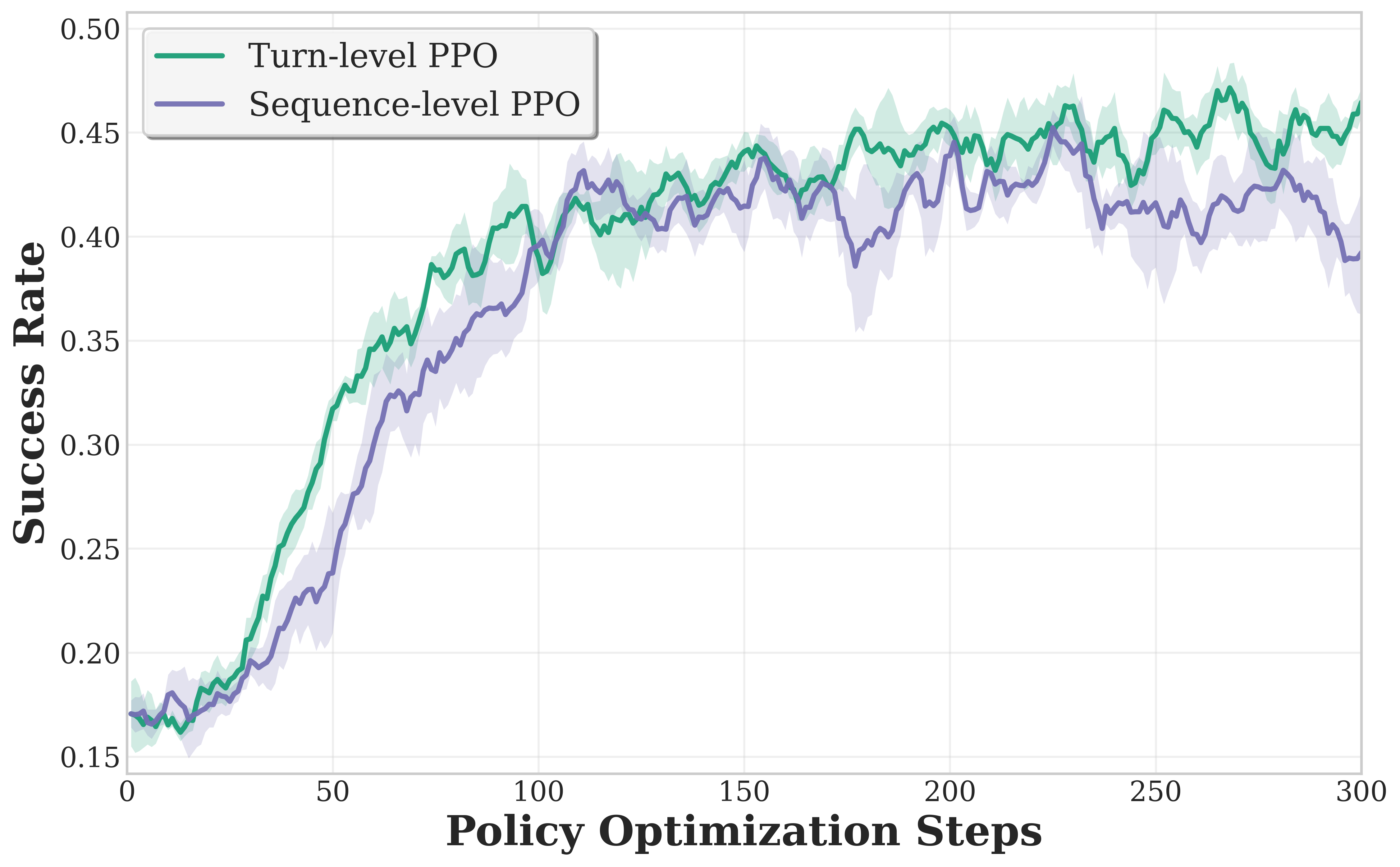}
        \caption{
        Turn-level PPO versus sequence-level PPO.
        }
        \label{fig:turn_vs_seq}
    \end{subfigure}
    \hfill
    \begin{subfigure}[t]{0.49\linewidth}
        \centering
        \includegraphics[
            width=\linewidth
        ]{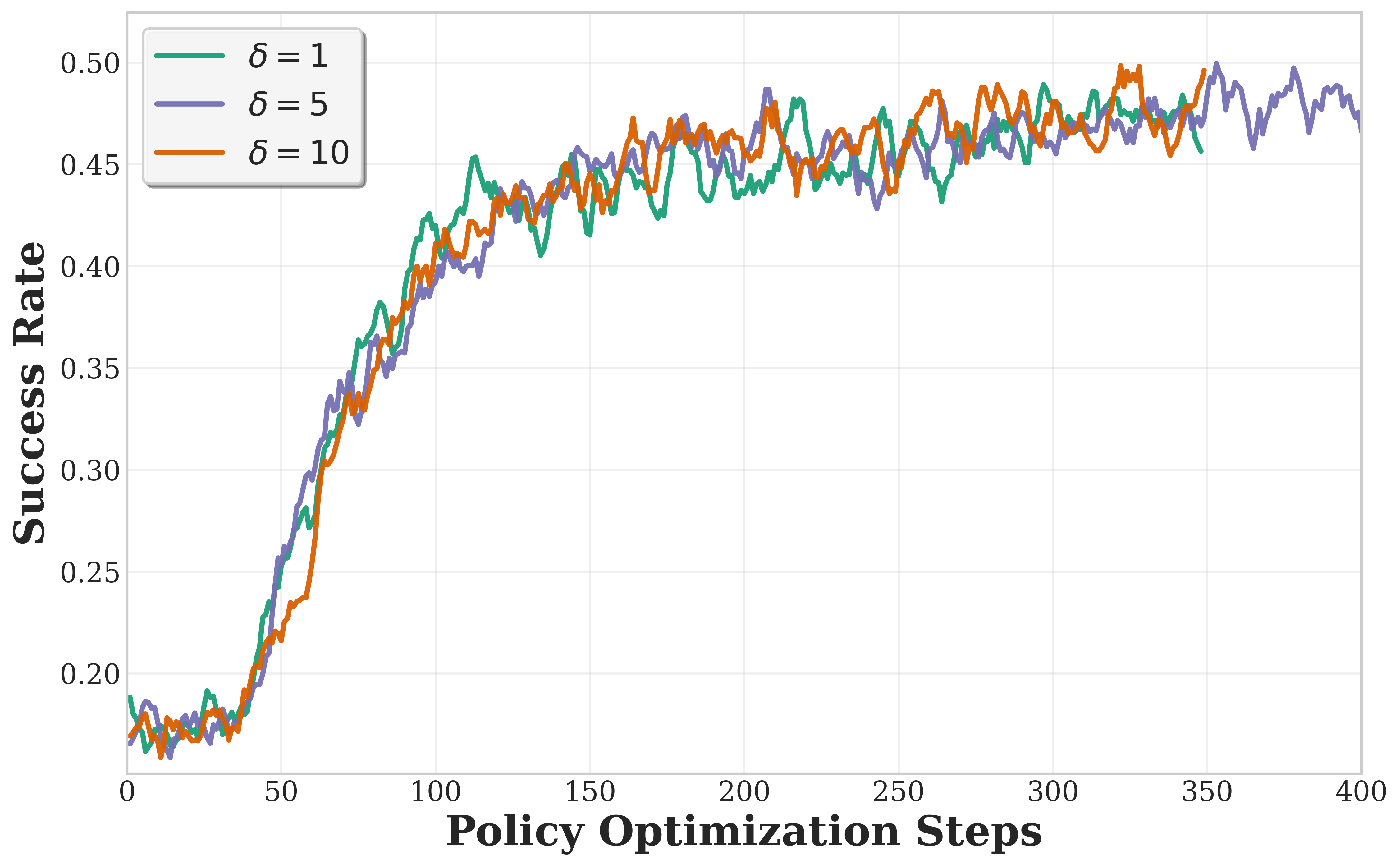}
        \caption{
        Sensitivity to the normalization threshold $\delta$.
        }
        \label{fig:delta_ablation}
    \end{subfigure}

    \caption{
    Ablations of the main design choices in \pponame{}.
    \textbf{Left:}  {Turn-level importance sampling} improves faster,
    reaches a higher final success rate, and exhibits more stable
    late-stage optimization than  {sequence-level importance weighting}.
    \textbf{Right:} The settings $\delta\in\{1,5,10\}$ produce similar
    optimization trajectories and comparable final performance.
    }
    \label{fig:design_ablations}
\end{figure}

\subsubsection{Additional Off-Policy Training Settings}
\label{app:more_offpolicy}

In this experiment, we set the training batch size to 512 and the
mini-batch size to 128, resulting in one on-policy update followed by
three increasingly off-policy updates for each sampled batch.

As shown in Figure~\ref{fig:more_offpolicy}, repeated mini-batch reuse
substantially destabilizes vanilla PPO. Although PPO initially achieves
competitive performance, its success rate degrades sharply in the
later stages of training. This degradation coincides with increasing
clipping ratios and KL divergence, indicating that the policy updates
become progressively more off-policy.

In contrast, \pponame{} maintains stable performance throughout training
and exhibits consistently lower clipping ratios and KL divergence.
These results demonstrate that the proposed method remains robust under
more aggressive sample reuse and is particularly effective in highly
off-policy training regimes.

\begin{figure}[H]
    \centering

    \begin{subfigure}[t]{0.45\linewidth}
        \centering
        \includegraphics[
            width=\linewidth
        ]{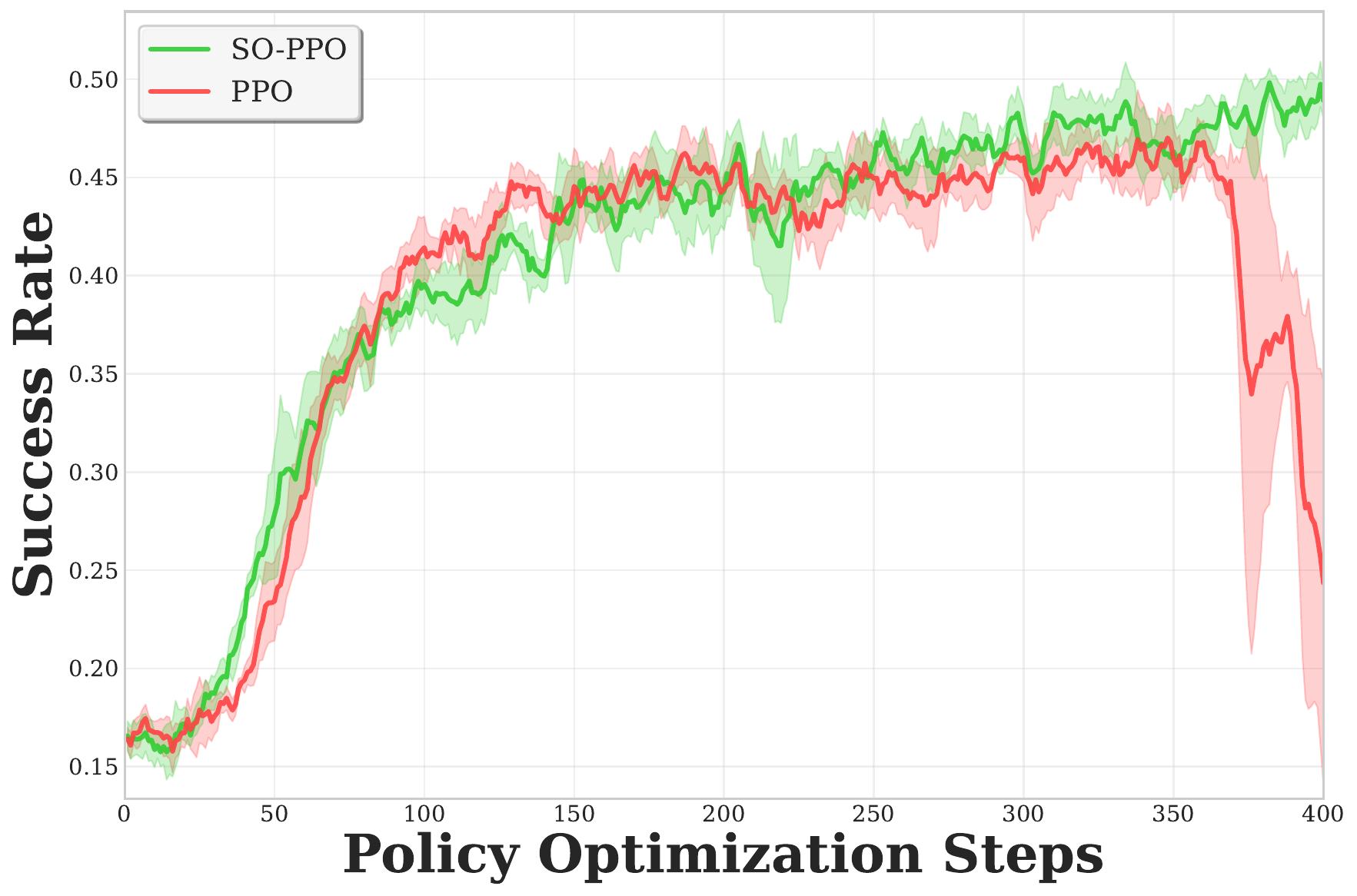}
        \caption{Success rate.}
        \label{fig:more_offpolicy_success}
    \end{subfigure}

    \vspace{0.15cm}

    \begin{subfigure}[t]{0.48\linewidth}
        \centering
        \includegraphics[
            width=\linewidth
        ]{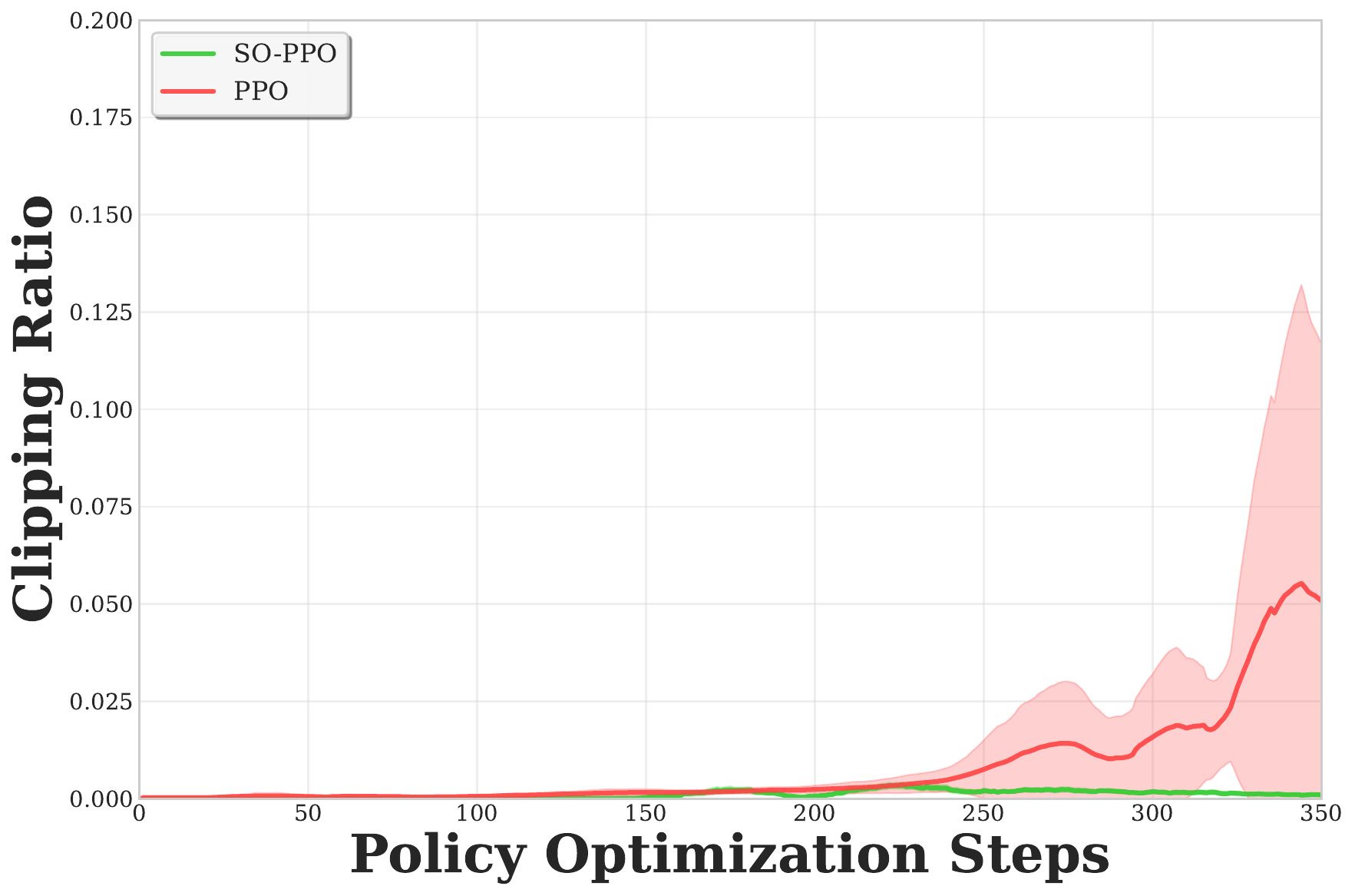}
        \caption{Clipping ratio.}
        \label{fig:more_offpolicy_clip}
    \end{subfigure}
    \hfill
    \begin{subfigure}[t]{0.48\linewidth}
        \centering
        \includegraphics[
            width=\linewidth
        ]{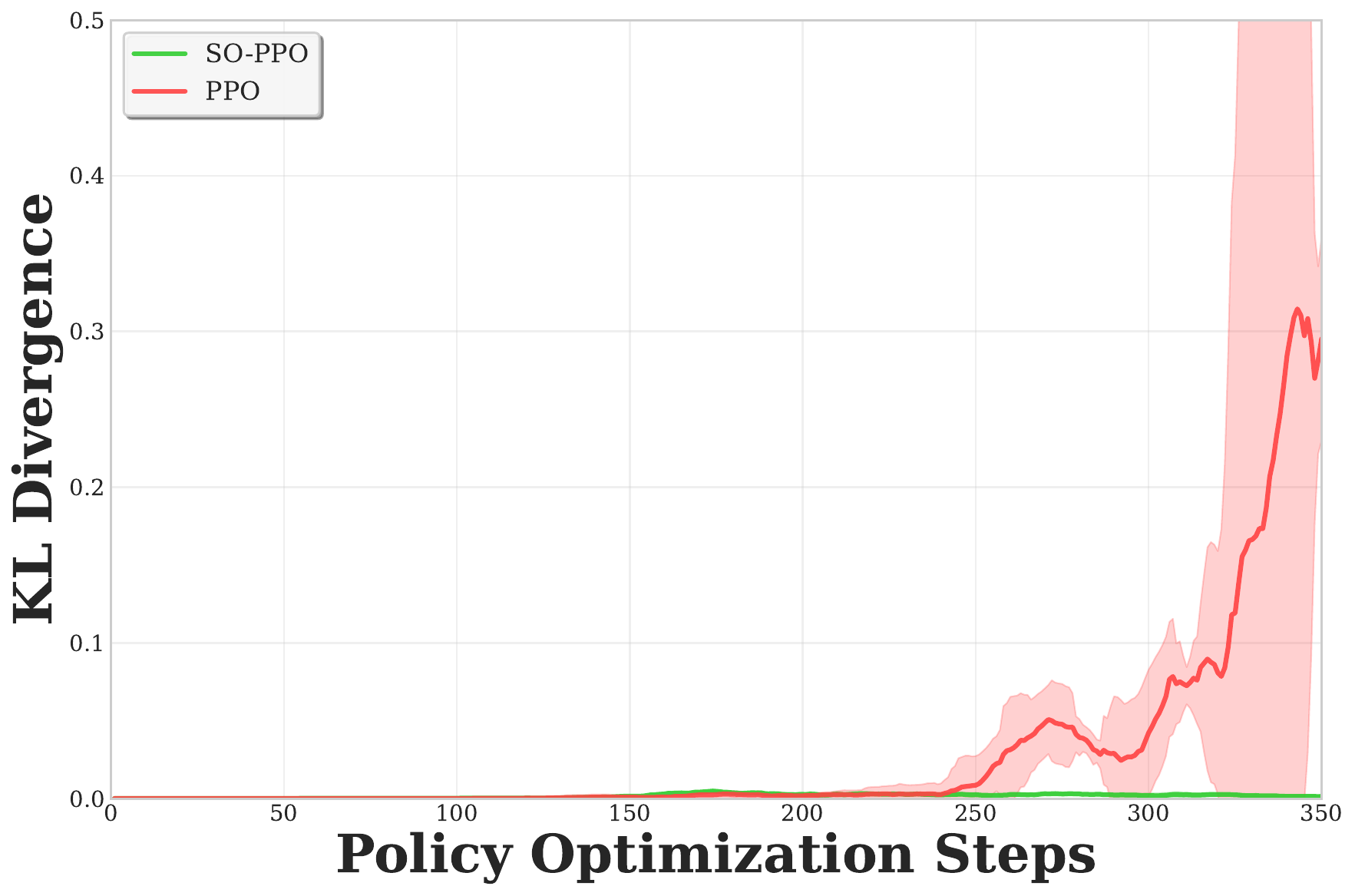}
        \caption{KL divergence.}
        \label{fig:more_offpolicy_kl}
    \end{subfigure}

    \caption{
    Training dynamics in a more off-policy setting with a training
    batch size of 512 and a mini-batch size of 128.
    \textbf{Top:} Success rate, averaged over three trials.
    \textbf{Bottom left:} Clipping ratio.
    \textbf{Bottom right:} KL divergence.
    Compared with vanilla PPO, \pponame{} maintains stable performance
    and exhibits lower clipping ratios and KL divergence.
    }
    \label{fig:more_offpolicy}
\end{figure}

\subsubsection{Effect of Optimizer Gradient Clipping}
\label{app:grad_clip_ablation}

Conventional optimizer gradient clipping limits the norm of the final
gradient update and is commonly used to prevent excessively large
optimization steps. However, it does not directly distinguish reliable
samples from highly off-policy samples. In contrast,
clipping-triggered normalization adjusts the update according to the
estimated  {clipping-bias norm} of the current  {mini-batch}.

To distinguish these two mechanisms, we independently enable and disable
optimizer gradient clipping for both PPO and \pponame{}.
As shown in Figure~\ref{fig:grad_clip_ablation}, standard PPO without
gradient clipping collapses shortly after reaching its peak performance.
Enabling gradient clipping delays this degradation but does not prevent
PPO from becoming unstable later in training.

In contrast, \pponame{} remains stable both with and without optimizer
gradient clipping, and the two configurations reach comparable final
performance. This result shows that the stabilization benefit of
\pponame{} does not rely on conventional gradient clipping.
Instead, clipping-triggered normalization stabilizes training by
adaptively reducing updates from batches dominated by unreliable
off-policy samples.

\begin{figure}[H]
    \centering
    \includegraphics[
        width=0.85\linewidth
    ]{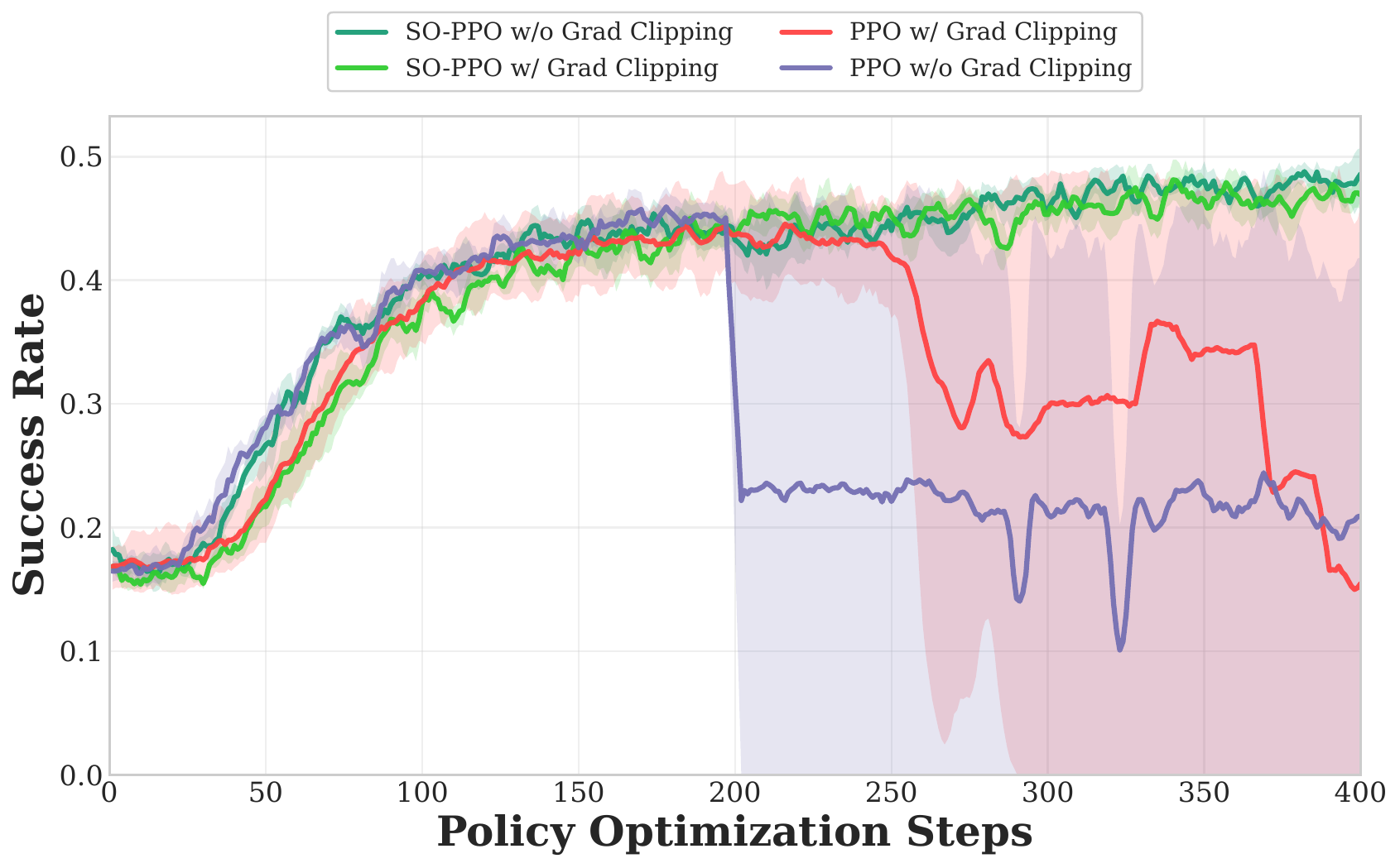}
    \caption{
    Effect of optimizer gradient clipping on NQ.
    Gradient clipping delays but does not prevent the late-stage
    degradation of standard PPO. In contrast, \pponame{} remains stable
    with or without gradient clipping, demonstrating that its
    stabilization benefit does not rely on conventional
    gradient-norm clipping.
    }
    \label{fig:grad_clip_ablation}
\end{figure}

\subsection{Implementation Details}
\label{app:implement_details}
\textbf{Retrieval System.}
A local retriever service is deployed and accessed via an HTTP interface. For each user query, we consistently return the top three passages. The dialogue is restricted to a maximum of three interaction turns.

\textbf{Training Process.}
All experiments are run on 8 NVIDIA H100 GPUs.
We activate gradient checkpointing and train under a Fully Sharded Data Parallel (FSDP) setup with offloading enabled for parameters, gradients, and optimizers.
The policy network is optimized with a learning rate of $1\times 10^{-6}$, while the critic is trained with $1\times 10^{-5}$.
{Optimization proceeds for up to four epochs; Table~\ref{tab:medical-benchmarks-safe-min} reports results from the selected medical-task checkpoints at 0.5 and 1 epoch of training.} The warm-up ratios are 0.285 (policy) and 0.015 (critic).
The effective batch size is 512, subdivided into mini-batches of 256 and micro-batches of 64 for policy updates and 8 for critic updates.
Generalized Advantage Estimation (GAE) is employed with $\lambda = 1$ and $\gamma = 1$.
Input sequences are truncated to at most 4{,}096 tokens, with limits of 500 tokens for responses, 2{,}048 tokens for the initial context, and 500 tokens for retrieved passages.
We adopt turn-level importance sampling combined with  {clipping-triggered normalization} to improve training stability.
Regularization follows PPO conventions with a KL coefficient of 0.001 and a clipping threshold of 0.2.
Rollouts are generated using vLLM with tensor parallel size 4, GPU memory utilization set to 0.6, temperature fixed at 1.0, and top-$p$ sampling of 1.0.
The rollout and reference log-probability calculations both use a micro-batch size of~128.

\textbf{Turn Boundary Identification.}
To implement turn-level importance sampling, we must identify turn boundaries within multi-turn trajectories. In our search task setup, we use the loss mask to distinguish between agent-generated content and environment responses: LLM-generated reasoning and query formulation steps are marked with 1 in the \texttt{loss\_mask}, while retrieved document content is marked with~0. We define turn boundaries by grouping consecutive tokens with \texttt{loss\_mask}~=~1 as complete turns (corresponding to the agent's actions in our turn-level MDP), while consecutive 0s represent states (i.e., retrieved content that provides environmental feedback). These identified turn boundaries enable our algorithm to apply turn-level importance sampling and credit assignment.

\textbf{Baselines for Medical Tasks.} We compare our method against several 8B-scale baselines commonly used in agentic tasks:
(1)~\textbf{RAG}: Retrieval-augmented generation \citep{glass2022re2g,li2025retrollm} using tool-following instructions, where external documents are retrieved to supplement the model’s responses. Note that the \mbox{Llama-3.1-8B-Instruct} model is used without additional training in this setting.
(2)~\textbf{ {CoT}}: Chain-of-Thought prompting~\citep{wei2022chain}, which encourages the model to generate explicit reasoning steps before producing an answer.
(3)~\textbf{Search-R1}: Search-R1 framework~\citep{jin2025search}, where we train the model using PPO and {report both the 0.5-epoch checkpoint and the best checkpoint observed within the first epoch, as listed in Table~\ref{tab:medical-benchmarks-safe-min}}.
 {(4)~\textbf{\pponame{} / \grponame{}}: the two algorithmic instantiations of our \aname{} framework.}

\textbf{Evaluation Datasets for Medical Tasks.} To comprehensively evaluate the algorithm's performance on medical tasks, we consider both in-domain and out-of-domain benchmarks. For in-domain evaluation, we use MedQA \citep{jin2020disease} and MedMCQA \citep{pal2022medmcqa}. For out-of-domain evaluation, we include PubMedQA \citep{jin2019pubmedqa}, the medical subset of MMLU \citep{wang2024mmlu}, and {MedXpertQA} \citep{zuo2025medxpertqa}. We exclude GPQA~\citep{rein2024gpqa} as it does not contain a medically relevant subset.

\subsection{Medical Prompt and Rollout Example}
\label{app:medical_prompt}

The prefix prompt below specifies the reasoning, retrieval, and answer tags used in every rollout on the medical tasks. It instructs the policy when to invoke search and constrains its final response to a single answer option. {The rollout example below illustrates how a trained policy follows this format by reasoning about the options, issuing a search query, incorporating retrieved evidence, and producing a final answer.}

\begin{tcolorbox}[title=Prefix prompt for medical tasks,
  colback=blue!5, colframe=blue!70!black, boxrule=1pt, arc=4pt,
  left=2mm, right=2mm, top=1mm, bottom=1mm, breakable]

Answer the given medical multiple-choice question. Think step-by-step inside <think> and </think> tags.

When you encounter:
\begin{itemize}
    \item Unfamiliar medical terminology or drug names
    \item Complex disease mechanisms you're uncertain about
    \item Specific treatment protocols or guidelines you need to verify
    \item Any information where you lack confidence
\end{itemize}

You can search for clarification using <search> query </search>, and results will appear between <information> and </information>.
Use search strategically to fill knowledge gaps and improve answer accuracy.
After sufficient reasoning and any necessary searches, provide your final answer inside <answer> and </answer> with ONLY the letter (e.g., <answer>A</answer>).

\end{tcolorbox}

\vspace{10mm}

\begin{tcolorbox}[
  title={A rollout example of answering a medical question},
  colback=blue!5, colframe=blue!70!black, boxrule=1pt, arc=4pt,
  left=2mm, right=2mm, top=1mm, bottom=1mm, breakable]
\small
\textbf{Question:} Which one of the following is not included in the treatment of malignant melanoma?

\textbf{Options:}

A: Radiation

B: Surgical excision

C: Chemotherapy

D: Immunotherapy

\tcblower

<think> Let's break down the question step by step! We're looking for the option that is not included in the treatment of malignant melanoma. </think>

\medskip

<think>
A: Radiation - Hmm, radiation therapy is sometimes used in combination with surgery or as palliative care for symptomatic relief. Not sure if it's the odd one out.

B: Surgical excision - Ah, surgery is a standard treatment for removing the primary tumor and any involved lymph nodes. Yes, that's a correct treatment option.

C: Chemotherapy - Chemotherapy can be used in some cases of malignant melanoma, particularly for patients with advanced disease or those who have recurrence after initial treatment. Yes, that's a valid treatment option.

D: Immunotherapy - Immunotherapy, such as pembrolizumab (Keytruda) or nivolumab (Opdivo), has shown significant efficacy in treating advanced melanoma by enhancing the body's immune response against cancer cells. Yes, that's a commonly used treatment option!

I'm getting confused about the radiation option, so let me search for clarification.
</think>

\medskip

{\color{cyan}<search>} What is the role of radiation therapy in treating malignant melanoma? {\color{cyan}</search>}

\medskip

{\color{gray}<information>}
Doc 1 (Melanoma): Radiation therapy is often used after surgical resection for locally or regionally advanced or unresectable metastases. It may reduce local recurrence but does not prolong survival.
Doc 2 (Radiation therapy): Melanoma is generally radioresistant; radiation is mainly palliative for metastatic disease.
Doc 3 (Radiation therapy): Radiation requires higher doses for melanoma than most cancers and is not curative. {\color{gray}</information>}

\medskip

<think>
Ah-ha! It seems that radiation therapy has a limited role in treating malignant melanoma and is mostly used in palliative care settings. It's not the first-line treatment for the disease. So, I think I can confidently say that...
</think>

\medskip

{\color{red}<answer>} A {\color{red}</answer>}
\end{tcolorbox}

\end{document}